\documentclass[12pt]{article}
\usepackage[utf8]{inputenc}
\usepackage{amsmath}
\usepackage{amssymb}
\usepackage{booktabs}
\usepackage{caption}
\usepackage{multirow}
\usepackage{tabularx}
\usepackage{graphicx}
\usepackage{verbatim}
\usepackage{subcaption}
\usepackage{authblk}
\usepackage{rotating,amsfonts,amsthm}
\usepackage{mathrsfs,color,graphics}
\usepackage{epsfig,array,wrapfig,longtable}
\usepackage{hyperref}
\usepackage{natbib}
\usepackage[normalem]{ulem}
\usepackage{algorithm, algorithmic}
\def\boxit#1{\vbox{\hrule\hbox{\vrule\kern6pt \vbox{\kern6pt#1\kern5pt}
\kern6pt\vrule}\hrule}}
\newtheorem{theorem}{Theorem}[section]

\newtheorem{definition}{Definition}[section]
\newtheorem{proposition}{Proposition}[section]
\newtheorem{corollary}{Corollary}[section]

\newtheorem{remark}{Remark}[section]

\graphicspath{{pics/}{}}
\DeclareGraphicsExtensions{.jpg,.pdf,.mps,.png}
\def\RR{\mathbb R}

\newcommand{\bo}{\mathbf{1}}

\newcommand{\bP}{\mathbf{P}}
\newcommand{\bS}{\mathbf{S}}
\newcommand{\bx}{\mathbf{x}}
\newcommand{\bX}{\mathbf{X}}

\newcommand{\by}{\mathbf{y}}
\newcommand{\bu}{\mathbf{u}}

\newcommand{\bz}{\mathbf{z}}
\newcommand{\bbeta}{{\boldsymbol{\beta}}}
\newcommand{\bmu}{{\boldsymbol{\mu}}}
\newcommand{\bgamma}{{\boldsymbol{\gamma}}}

\newcommand{\bet}{{\boldsymbol{\eta}}}
\newcommand{\bPi}{{\boldsymbol{\Pi}}}
\newcommand{\bsigma}{{\boldsymbol{\sigma}}}
\newcommand{\bSigma}{{\boldsymbol{\Sigma}}}

\begin{document}
\def\spacingset#1{\renewcommand{\baselinestretch}%
{#1}\small\normalsize} \spacingset{1}
% Title of paper
\title{A Novel Framework for Online Supervised Learning with Feature Selection}
\date{\vspace{-5ex}}

% List of authors, with corresponding author marked by asterisk

\author{Lizhe Sun, Mingyuan Wang, Siquan Zhu and Adrian Barbu\thanks{To whom correspondence should be addressed: Adrian Barbu. Barbu is Professor, Department of Statistics, Florida State University, Tallahassee, FL 32306. Sun is Postdoc, Beijing International Center for Mathematical Research, Peking University. Wang is Ph.D. graduate and Zhu is Ph.D. student, Department of Statistics, Florida State University.}}

%\author{Anonymous, Anonymous and Anonymous}
\maketitle

\begin{abstract}
Current online learning methods suffer issues such as lower convergence rates and limited capability to select important features compared to their offline counterparts. 
In this paper, a novel framework for online learning based on running averages is proposed.
Many popular offline regularized methods such as Lasso, Elastic Net, Minimax Concave Penalty (MCP), and Feature Selection with Annealing (FSA) have their online versions introduced in this framework. 
The equivalence between the proposed online methods and their offline counterparts is proved, and then novel theoretical true support recovery and convergence guarantees are provided for some of the methods in this framework. 
Numerical experiments indicate that the proposed methods enjoy high true support recovery accuracy and a faster convergence rate compared with conventional online and offline algorithms.  
Finally, applications to large datasets are presented, where again the proposed framework shows competitive results compared to popular online and offline algorithms.

\end{abstract}

%%%%%%%%%%%%%%%%%%%%%%%%%%%%%%%%%%%%%%%%%%%%%%%%%%%%%%%%%%%%
%%%% Main text entry area:
%%%%%%%%%%%%%%%%%%%%%%%%%%%%%%%%%%%%%%%%%%%%%%%%%%%%%%%%%%%%
\newpage
\spacingset{1.8}
\section{Introduction}\label{sec1:introduction}

Online learning is one of the most promising research areas that can efficiently handle large-scale data analysis problems. 
In many scenarios, since observations are coming in real-time, it is not easy to collect all observations together and then learn a model.
Moreover, due to datasets from various areas such as bioinformatics, medical imaging, and computer vision rapidly increasing in size, we often encounter the problem that datasets are so large that they cannot fit in the computer memory.
Learning statistical models from such large-scale datasets needs a large amount of computational and stored resources.
The conventional online framework can address these issues by constructing and updating the model sequentially, using one example at a time, or a mini-batch of examples at a time. 
There are a sequence of observations $\bz_i = (\bx_i, y_i)$, $i = 1, 2,...$, that are coming, where $\bx_i \in \RR^p$ is a $p$-dimensional vector and $y_i \in \RR$. Consider the coming $n$-th observation $\bz_n$ and the current model coefficient vector $\bbeta_{n} \in \RR^p$, the updated $\bbeta_{n+1}$ can be learned sequentially by updating the gradient 
\begin{equation*}
\bbeta_{n+1} = \bbeta_{n} - \eta\frac{\partial f(\bbeta,\bz_n)}{\partial \bbeta},
\end{equation*}
where $f(\cdot; \bz_n):\RR^p \to \RR$ is a per-example loss function and $\eta$ is a learning rate.
The elements in the vector $\bbeta_1 = \mathbf{0}$ are the initialized coefficients. 
In the theoretical analysis of online learning, it is of interest to obtain an upper bound of the regret 
\begin{equation*}
R_n=\frac{1}{n}\sum_{i=1}^{n} f(\bbeta_i;\bz_i) - \min_{\bbeta}\frac{1}{n}\sum_{i=1}^{n}f(\bbeta;\bz_i).
\end{equation*} 
The regret can measure the difference of the loss compared to an offline optimization algorithm, and the speed of convergence of the online algorithms.
Under the assumptions that $f(\bbeta; \bz_i)$ is Lipschitz-continuous and convex w.r.t $\bbeta$, the regret enjoys the upper bound of $\mathcal{O}(1/ \sqrt{n})$ \citep{zinkevich2003online}. 
Moreover, if $f(\bbeta; \bz_i)$ is a strongly convex function, the regret has the logarithmic upper bound of $\mathcal{O}(\log(n)/n)$ \citep{hazan2007logarithmic}.

Several online methods are proposed to solve the variable selection problem in the conventional online scenario. 
For online convex optimization, there are two main lines of research. 
One is the Forward-Backward-Splitting method \citep{duchi2009efficient}, building a framework for online proximal gradient (OPG). 
The other one is Xiao's Regularized Dual Averaging method (RDA) \citep{xiao2010dual}, which extended the primal-dual sub-gradient method from \cite{nesterov2009primal} to the online case. 
Some online variants were developed in recent years, such as OPG-ADMM and RDA-ADMM \citep{suzuki2013dual}. Independently, \cite{ouyang2013stochastic} designed stochastic ADMM (SADMM) as well, the same algorithm as OPG-ADMM.
Additionally, for online non-convex optimization, some methods based on the $\ell_0$ penalty were proposed in the recent literature.
\cite{langford2009sparse} proposed a variant of the truncated stochastic gradient descent (TSGD) method.
Then \cite{fan2018statistical} provided a statistical analysis of the truncated stochastic gradient descent method.
Similar methods were also proposed in \cite{Wang2014OnlineFS, wu2017large, nguyen2017linear}.
Additionally, a Bayesian truncated stochastic gradient descent method was proposed in \cite{Cai2009OnlineFS}. 
In this paper, the methods in the proposed framework are compared with the SADMM method \citep{ouyang2013stochastic} and the TSGD method \citep{langford2009sparse} in the linear regression model.  

There is another research direction for online feature selection in the high-dimensional case. 
\cite{yang2016online} proposed a new framework for online learning in which features arrive one by one, instead of observations, and then we decide what features to retain. 
Unlike the conventional online learning case, the disadvantage of this new online scenario is that we cannot learn a model for prediction until all relevant features are disclosed. 
In this article, we assume that the observations arrive sequentially in time, therefore we do not cover algorithms such as \cite{yang2016online} for comparison.    

However, conventional online methods have some limitations. 
First, they cannot access the full gradient to update the parameter vector at each iteration. 
Online methods are sequential methods, using one observation or a mini-batch for acceleration \citep{cotter2011better} at each iteration. 
Therefore, online methods such as online gradient descent (OGD) suffer a lower convergence rate, $\mathcal{O}(1/\sqrt{n})$ for general convexity and $\mathcal{O}(\log(n)/n)$ for strongly convex functions \citep{shalev2014understanding}. 
In contrast, conventional offline methods enjoy the convergence rate of $\mathcal{O}(1/n)$ \citep{gyorfi2002distribution}. 
Second, they are not able to exploit the sparse structure of the coefficient vector, i.e., they cannot select variables.
Although the existing OPG and RDA methods can induce a sparse estimated coefficient vector, they cannot recover the support of true signals.  

%%%%%%%%%%%%%%%%%%%%%%%%%%%%%%%%%%%%%%%%%%%%%%%%%%%%%%%%%%%%%%%%%%%%%%%%%%%%%%%%%%%%%%%%%%%%%%%%%%%%%%%%%%%%%%%%%%
\begin{table}[t] 
\begin{center}
\caption{Overview of the proposed and conventional online methods and their capabilities. }\label{table1:algorithms}
\scalebox{0.75}{
\begin{tabular}{c|ccccccc}
\hline
&Memory&\multicolumn{2}{c}{Computation time} &\multicolumn{2}{c}{Convergence rate} &Feature &True Support    \\
Methods   &                     &Running Averages                   &Algorithms                   &Methods                  &Regret                                 &Selection  &Recovery \\
\hline
OGD          &$\mathcal{O}(p)$       &-                               &$\mathcal{O}(np)$         &$\mathcal{O}(n^{-1/2})$ &Slow                                   &No                         &No \\
\hline
SADMM    &$\mathcal{O}(p)$       &-                                         &$\mathcal{O}(np)$          &$\mathcal{O}(n^{-1/2})$ &Slow                                  &Yes                         &No \\
\hline
TSGD         &$\mathcal{O}(p)$       &-                                         &$\mathcal{O}(np)$          &$\mathcal{O}(n^{-1/2})$ &Slow                                  &Yes                        &No \\
\hline
OFSA (ours)      &$\mathcal{O}(p^2)$   &$\mathcal{O}(np^2)$           &$\mathcal{O}(p^2)$        &$\mathcal{O}(n^{-1})$   &Fast                                   &Yes                        &Yes \\
\hline
OLS-th (ours)      &$\mathcal{O}(p^2)$   &$\mathcal{O}(np^2)$           &$\mathcal{O}(p^3)$      &$\mathcal{O}(n^{-1})$   &$\mathcal{O}(\log^2(n)/n)$                    &Yes                        &Yes  \\
\hline
OMCP (ours)      &$\mathcal{O}(p^2)$   &$\mathcal{O}(np^2)$            &$\mathcal{O}(p^2)$        &$\mathcal{O}(n^{-1})$       &Fast                                  &Yes                         &Yes \\
\hline
OElnet (ours)     &$\mathcal{O}(p^2)$   &$\mathcal{O}(np^2)$            &$\mathcal{O}(p^2)$        &$\mathcal{O}(n^{-1})$       &Fast                                   &Yes                         &Yes \\
\hline
\end{tabular}}
\end{center}
\end{table}

This article proposes a novel framework for online learning with feature selection using running averages (RAVEs).
The proposed framework uses the idea in the Recursive Least Squares \citep{kushner2003stochastic} method, which can update $\bX^T\bX$ incrementally to solve the least squares problem. However, it goes beyond the least squares in the sense that it shows how to use the RAVEs to adapt many feature selection algorithms such as Lasso and other penalized methods, as well as Feature Selection with Annealing to online learning. Moreover, it addresses under the same framework the problem of binary classification with imbalanced data.

\textcolor{black}{The proposed novel framework is designed to address the streaming data. In this case, the sample size $n$ may change from a very small value to approach infinity while the number of predictors $p$ remains at a smaller order than $n$. Therefore, the methods in the proposed novel framework may need to address both the low-dimensional and high-dimensional cases. In real-world data analysis, the proposed methods in this framework are well equipped to handle data with the number of predictors up to roughly 50k, and the sample size $n$ can be arbitrarily large. Although the sample size $n$ is so large, the online methods in this framework can learn a model with a constant memory requirement.
Additionally, the proposed methods in this framework can learn different models with different sparsity levels and even different types of penalties at any time.}
Table \ref{table1:algorithms} summarizes the convergence rate and the capability of variable selection for the proposed methods and the existing methods.

It is worth noting that several papers have discussed the topic of streaming data during the same period as this work and the RAVEs methodology was utilized in these publications \citep{Schifano2015OnlineBigData, Luo2019RenewableEA, luo2023multivariate}.   
These articles called running average statistics as cumulative summary statistics (CUSUM), which was originally introduced by \cite{Page1955ATF}. 
Furthermore, RAVEs have proven to be applicable in various topics, including statistical query model \citep{kearns1998efficient, chu2007map}, change point detection \citep{Yu2017FiniteSC}, and genetic correlation estimation \citep{Wang2021EstimationGC}. 
These examples show the usefulness of the RAVEs approach in diverse research areas.
However, RAVEs are only used for statistical inference in these papers.
They did not discuss the application of RAVEs to variable selection in the framework of online learning.
To our knowledge, this is the first work to use RAVEs for online variable selection and concept drifting data. \textcolor{black}{And the theoretical analysis for the proposed online methods is provided under the assumption of the linear regression model.} Compared to the existing online feature selection methods such as OPG, RDA, and TSGD, the proposed methods offer superior performance on variable selection and prediction for future data.

The remaining parts of the paper are organized as follows. 
Section \ref{sec2:setandnota} introduces the novel online learning framework.
Section \ref{sec3:method} proposes online versions of many popular variable selection methods in this framework.
Section \ref{sec4:theory} provides the theoretical guarantees for the proposed methods in this framework. 
Section \ref{sec5:simandrealdata} assesses the performance of the proposed methods via simulation studies and some real data analyses. 
Section \ref{sec6:discuss} concludes the paper with a brief discussion.

\section{Setup and Notation}\label{sec2:setandnota}
In this section, a novel framework using running averages (RAVEs) is developed.
First, we establish notation and problem settings. 
We denote vectors by lowercase bold letters, such as $\bx \in \mathbb{R}^d$, and scalars by lowercase letters, e.g. $x \in \mathbb{R}$. 
A sequence of vectors is denoted by subscripts, i.e., $\mathbf{w}_1, \mathbf{w}_2, \dots$, and the entries in a vector are denoted by non-bold subscripts, like $w_j$. 
We use uppercase bold letters to denote matrices, such as $\mathbf{M} \in \mathbb{R}^{d \times d}$, and uppercase letters for random variables, like $X$.
For a vector $\bgamma = (\gamma_1, \gamma_2, \cdots, \gamma_d)^T \in \mathbb{R}^d$, we define the vector norms:
$\|\bgamma\|_1 = \sum_{j=1}^{d}\mid\gamma_j\mid$ and $\|\bgamma\| = \sqrt{\sum_{j=1}^{d}\gamma_j^2}$.

\subsection{Running Averages}
Let $\{(\bx_i, y_i)\}_{i=1}^n$ be observations with $\bx_i = (x_{i1}, x_{i2}, \cdots, x_{ip})^T \in \mathbb{R}^p$ and $y_i \in \mathbb{R}$, $\bX = (\bx^T_1, \bx^T_2, \cdots, \bx^T_n)^T$ be data matrix, and $\mathbf{y} = (y_1, y_2, \cdots, y_n)^T$ be response.
Then the RAVEs are denoted as follows by
\begin{eqnarray*}
&&{\bmu}_{x} = \frac{1}{n}\sum_{i=1}^{n}\bx_i, \;\bS_{xx} = \frac{1}{n}\sum_{i=1}^{n} \bx_i\bx^T_i, \;
\bS_{xy} = \frac{1}{n}\sum_{i=1}^{n}y_i\bx_i, \\
&&{\mu}_{y} = \frac{1}{n}\sum_{i = 1}^{n}y_i, \; S_{yy} = \frac{1}{n}\sum_{i=1}^{n}y^2_i,
%&&{\bmu}_{x} = \frac{1}{n}\sum_{i=1}^{n}\bx_i, \; {\mu}_{y} = \frac{1}{n}\sum_{i = 1}^{n}y_i, \\
%&&\bS_{xx} = \frac{1}{n}\sum_{i=1}^{n} \bx_i\bx^T_i, \; \bS_{xy} = \frac{1}{n}\sum_{i=1}^{n}y_i\bx_i, \; S_{yy} = \frac{1}{n}\sum_{i=1}^{n}y^2_i, 
\end{eqnarray*}
and the cumulative sample size $n$. These RAVEs can be updated incrementally such as
\begin{equation}
\bmu_x^{(n+1)}=\frac{n}{n+1} \bmu_x^{(n)}+\frac{1}{n+1}\bx_{n+1}, \label{eq:raveupdate}
\end{equation}
which is the procedure from Chapter 11.2.3 of \cite{pml1Book}.
A diagram of the RAVEs framework is shown in Figure \ref{fig: diagram}.
\begin{figure}[ht]
\centering
\includegraphics[width=0.75\linewidth]{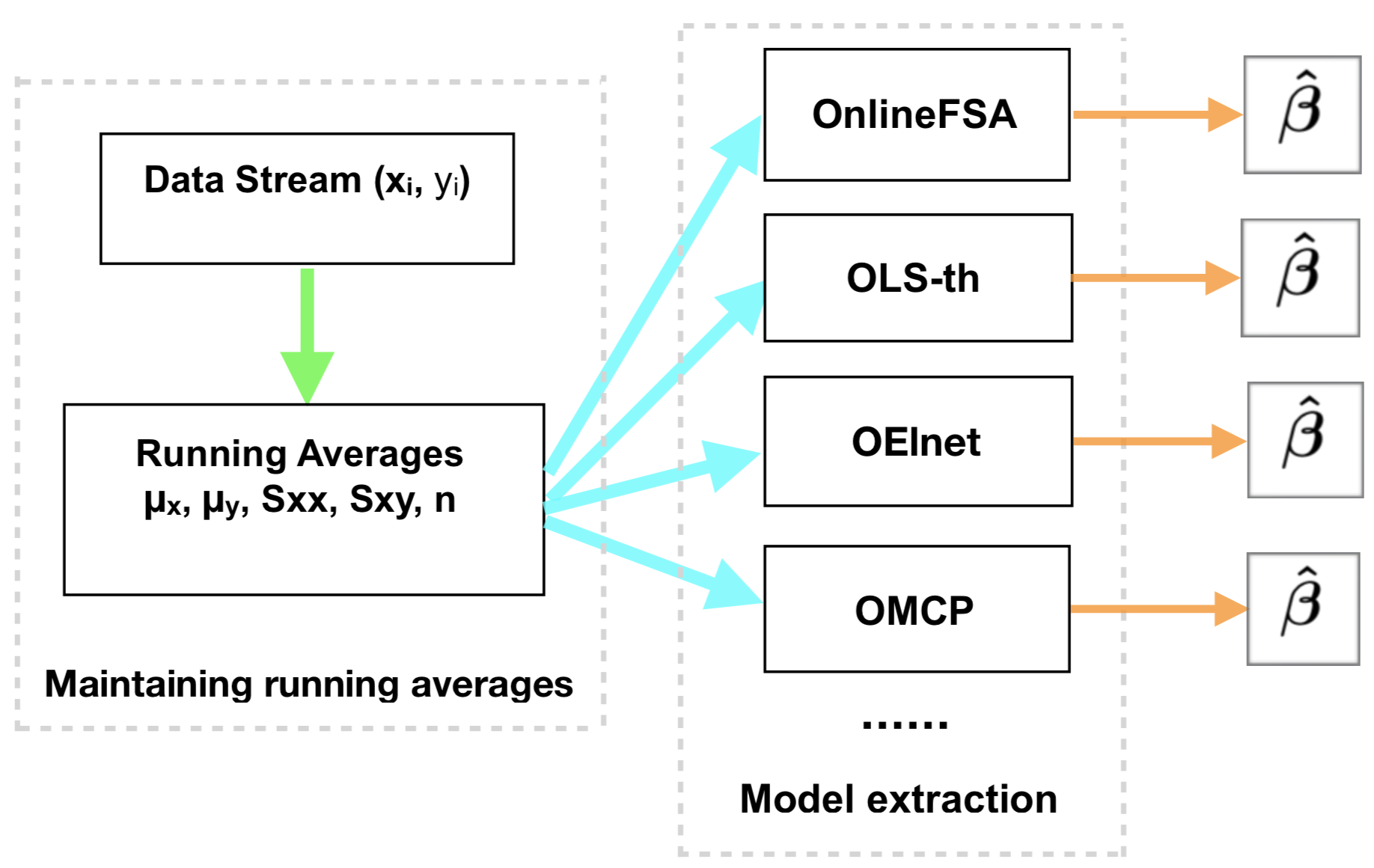}
\caption{Diagram of the running averages-based methods. The RAVEs are updated as the data is received. The model can be extracted from the running average statistics at any time.}\label{fig: diagram}
\end{figure}

Compared to the conventional model training methods, learning a statistical model by RAVEs has the following advantages.
First, the RAVEs such as $\bS_{xx}= \frac{1}{n}\sum_{i=1}^{n} \bx_i\bx^T_i$ and $\bS_{xy} = \frac{1}{n}\sum_{i=1}^{n}y_i\bx_i$ can be computed on several machines, assigning a batch of observations to one machine.
Therefore, the partial summary statistics for each batch can be computed in parallel, and then RAVEs can be reconstructed from these partial summary statistics. 
Moreover, the RAVEs contain all the necessary sample information for learning certain models.
Additionally, the dimension of the running average statistics will not change with sample size $n$.
The RAVEs can be used in the scenario of online learning to address the streaming data since they are updated one example at a time.

\subsection{Data Standardization}\label{subsec2:Standardization}
Data standardization is an important step in real data analysis, especially for feature selection, because a feature could have an arbitrary scale (unit of measure) and its scale should not influence its importance in the model. 
For this purpose,  the data matrix $\bX$ and the response vector $\by$ are usually standardized by removing the mean, and $\bX$ is further standardized by bringing all columns to the same scale. 
However, we need to standardize the running average statistics because the data is discarded in the online setup and only the RAVEs are memorized. 
Separate steps will be shown for subtracting a mean $\bmu_x$ and for dividing by a standard deviation vector $\bsigma$.

\noindent {\bf Subtracting a mean} $\bmu_x$.
Denoting by $\mathbf{1}_n = [1,1,\cdots,1]^T \in \mathbb{R}^n$, the standardized version of the data matrix $\bX$ and the response $\by$ are: 
\begin{equation*}
\mathbf{\bar X} = \bX - \mathbf{1}_n\bmu_x^T, \ \mathbf{\bar y} = \by - \mu_y\mathbf{1}_n.
\end{equation*}
Then the mean of RAVEs can be removed by following steps:
\begin{equation*}
\begin{split}
    \bS_{\bar x \bar y} &= \frac{1}{n}(\bX - \mathbf{1}_n\bmu_x^T)^T(\by - \mu_y\mathbf{1}_n) \\
    &= \frac{1}{n}\bX^T\by - 2\mu_y\bmu_x + \mu_y\bmu_x = \bS_{xy} - \mu_y\bmu_x, \\
    \bS_{\bar x\bar x}&=\frac{1}{n}(\bX - \mathbf{1}_n\bmu_x^T)^T(\bX - \mathbf{1}_n\bmu_x^T) \\
    &= \frac{1}{n}\bX^T\bX - \bmu_x \bmu^T_x - \bmu_x \bmu_x^T+ \bmu_x \bmu_x^T = \bS_{xx} - \bmu_x \bmu^T_x.
\end{split}
\end{equation*}
\noindent {\bf Dividing by} $\bsigma=(\sigma_1,...,\sigma_p)^T$.
Let $\bPi=\text{diag}(\sigma_1,...,\sigma_p)^{-1}$ be the $p \times p$ diagonal matrix containing the inverse of the $\sigma_j$ on the diagonal. 
The standardized dataset is denoted as
$\mathbf{\tilde X}=\bX\boldsymbol{\Pi}$ and the RAVEs of the standardized dataset are:
\begin{gather*}
\bS_{\tilde x y}=
\frac{1}{n}\mathbf{{\tilde X}}^T\mathbf{ y} = 
\frac{1}{n}\boldsymbol{\Pi}\mathbf{X}^T\mathbf{y}
=\boldsymbol{\Pi}\bS_{xy}, \\
\bS_{\tilde x\tilde x}=\frac{1}{n}\mathbf{{\tilde X}}^T\mathbf{{\tilde X}}
=\boldsymbol{\Pi}\frac{\bX^T\bX}{n}\boldsymbol{\Pi} 
= \boldsymbol{\Pi}\bS_{xx} \boldsymbol{\Pi}.
\end{gather*}
The sample standard deviation for the random variable $X_j$ can be estimated by RAVEs as
\begin{equation*}
\sigma_{x_j} = \sqrt{(\bS_{xx})_j - (\bmu_{x})_j^2},
\end{equation*}
where $(\bS_{xx})_j$ is the $j$-th diagonal entry of $\bS_{xx}$. For convenience, in the rest of this article, we still use $\bS_{xx}$ and $\bS_{xy}$ to represent the RAVEs after the standardization procedures.

%%%%%%%%%%%%%%%%%%%%%%%%%%%%%%%%%%%%%%%%%%%%%%%%%%%%%%%%%%%%%%%%%%%%%%%%%%%%%%%%%%%%%%%%%%%%%%%%%%%%%%%%%%%%%%%%%%%%%%%%%%%%%%%%%%%%%%
\section{Methods}\label{sec3:method}
Several running averages-based online learning methods are proposed in this section.
These proposed online methods can address the low-dimensional and high-dimensional settings.
The applications of the proposed methods in the low-dimensional setting are mainly discussed, even though some of them can also address the high-dimensional cases. 
The first is the online version of the ordinary least squares method.
Then, the online least squares with thresholding method is proposed, which can select the important features.
Finally, online versions of many feature selection methods, including Feature Selection with Annealing (FSA), Lasso, Elastic Net, and Minimax Concave Penalty (MCP) are introduced.  
To simplify the notations, denote online least squares by OLS, OLSth for OLS with thresholding, OFSA for online FSA,  OLasso for online Lasso, OElnet for online Elastic Net, and OMCP for online MCP. 

\subsection{Preliminaries}
First, we prove that these online methods are equivalent to their offline counterparts. 
In fact, in the RAVEs framework, the proposed methods share the same objective loss function as in offline learning, which is the key point to prove their equivalence.
\begin{proposition}\label{prop1}
Consider the general penalized regression problem
\begin{equation}\label{eq:penreg}
\textcolor{black}{\min_{\bbeta \in \mathbb{R}^p}}\frac{1}{2n}\|\mathbf{y} - \bX\bbeta\|^2 + \bP(\bbeta; \lambda),
\end{equation}
in which $\bbeta \in \mathbb{R}^p$ is the coefficient vector and $\bP({\bbeta}; \lambda) = \sum_{j=1}^{p}\bP(\beta_j ; \lambda)$ is a penalty function. 
It is equivalent to the online optimization problem based on RAVEs by 
\begin{equation}\label{eq:rsreg}
\textcolor{black}{\min_{\bbeta \in \mathbb{R}^p}} \frac{1}{2}\bbeta^T\bS_{xx}\bbeta - \bbeta^T\bS_{xy} + \bP(\bbeta; \lambda).
\end{equation}
\end{proposition}
\vspace{-3mm}
\begin{proof}
The loss function (\ref{eq:penreg}) can be rewritten as
\begin{equation*}
\begin{split}
     \frac{1}{2n}\|\mathbf{y}  - \bX\bbeta\|^2 + \bP(\bbeta; \lambda) =& \frac{1}{2n}(\mathbf{y} - \bX\bbeta)^T(\mathbf{y} - \bX\bbeta) + \bP(\bbeta; \lambda) \\
    =& \frac{\mathbf{y}^T\mathbf{y}}{2n} - \frac{\bbeta^T\bX^T\mathbf{y}}{n} + \bbeta^T\frac{\bX^T\bX}{2n}\bbeta+ \sum_{j=1}^{p}\bP(\beta_j; \lambda),
\end{split}
\end{equation*}
where $S_{yy}=\mathbf{y}^T\mathbf{y}/n$, $\bS_{xy}=\bX^T\mathbf{y}/n$, and $\bS_{xx}=\bX^T\bX/n$ are running averages. Therefore, the offline learning optimization problem (\ref{eq:penreg}) is equivalent to the running averages-based optimization problem (\ref{eq:rsreg}).  
\end{proof}

\subsection{Online Least Squares}
In the ordinary least squares method, the coefficient vector $\bbeta$ is estimated by solving $\bX^T\bX\bbeta = \bX^T\mathbf{y}$. 
Since $\bX^T\bX$ and $\bX^T\mathbf{y}$ can be computed by using RAVEs, we can obtain
\begin{equation*}
\bS_{xx} \bbeta= \bS_{xy} .
\end{equation*}
Thus, the online least squares method is equivalent to offline least squares.
%%%%%%%%%%%%%%%%%%%%%%%%%%%%%%%%%%%%%%%%%%%%%%%%%%%%%%%%%%%%%%%%%%%%%%%%%%%%%%%%%%%%%%%%%%%%%%%
\subsection{Online Least Squares with Thresholding}
OLS with thresholding (OLSth) is a simple method aimed at solving the following constrained minimization problem
\begin{equation*}
\textcolor{black}{\min_{\bbeta \in \mathbb{R}^p, \|\bbeta\|_0 \leq k^*}} \dfrac{1}{2n} \|\mathbf{y} - \bX\bbeta\|^2,
\end{equation*}
which is a non-convex and NP-hard problem because of the sparsity constraint. 
A three-step procedure is proposed to find an approximate solution: 
first, the online least squares method is used to estimate $\hat\bbeta$.
Then the important variables are selected according to the coefficient magnitudes $|\beta_j|$, $j = 1, 2, \cdots, p$. 
Finally, the least squares method is used to refit the model on the subset of selected features. The prototype method is described in Algorithm \ref{alg:olsth}.
If the number of true variable $k^*$ is unknown, the important covariates can be selected by the hypothesis test $H_0: \beta_j = 0$ for each $j = 1,2,\cdots, p$, where a $t$-statistic or $z$-statistic may be used.

\begin{algorithm}
\caption{\bf \small{Online Least Squares with Thresholding (OLSth)}}\label{alg:olsth}
\begin{algorithmic}[]
\STATE {\bfseries Input:} RAVEs $\bS_{xx}, \bS_{xy}$, sample size $n$, and true sparsity level $k^*$.
\STATE {\bfseries Output:} Coefficient vector $\bbeta$ with $\|\bbeta\|_0\leq k^*$.
\begin{enumerate}
    \item Find $\hat{\bbeta}$ by OLS.
    \item Keep only $k^*$ variables with largest $\lvert\hat{\beta}_j\rvert$.
    \item Fit the model on the selected features by OLS.
\end{enumerate}
\end{algorithmic}
\end{algorithm}
\begin{remark}\label{remark1}
In the high dimensional scenario ($p > n$), the estimator $\hat{\bbeta}$ in step 1 of the Algorithm \ref{alg:olsth} can come from online ridge regression rather than the OLS estimator.
If the true number $k^*$ is unknown, the multiple subset regression based on RAVEs can be run for high-dimensional statistical inference \citep{Liang2022MNR}, then the important features can be selected by using the $p$-values from the multiple hypothesis tests.
\end{remark}

%%%%%%%%%%%%%%%%%%%%%%%%%%%%%%%%%%%%%%%%%%%%%%%%%%%%%%%%%%%%%%%%%%%%%%%%%%%%%%%%%%%%%%%%%%%%%%%%%%%%%%%%%%%%%%%%%%%%%%%%%%%%%%%%
\subsection{Online Feature Selection with Annealing}
Online Feature Selection with Annealing (OFSA) is an iterative thresholding method. 
The OFSA method can simultaneously solve the coefficient estimation problem and the feature selection problem. 
The main procedures in OFSA are: 1) uses an annealing plan to lessen the greediness in reducing the dimensionality from $p$ to $k^*$, 
2) removes irrelevant variables to facilitate computation.
The algorithm starts with an initialized parameter $\bbeta$, generally $\bbeta = 0$, and then alternates two basic steps: one is updating the parameters to minimize the loss $L(\bbeta)$ by gradient descent
$$
\bbeta = \bbeta - \eta\frac{\partial L}{\partial \bbeta},
$$
and the other one is a feature selection step that removes some variables based on the ranking of $\lvert\beta_j\rvert$, $j = 1, 2, \cdots, p$.
In the second step, an annealing schedule is used to decide the number of features $M_t$ kept at step $t$
$$
M_t = k^* + (p - k^*)\max\{0, \frac{T - t}{t\mu + T}\}, t = 1, 2, \cdots, T,
$$
where $T$ is the total iteration times.
Observe that these iteration steps $t = 1, 2,\cdots$ have nothing to do with the observations $n$.
More details are shown in \cite{barbu2017feature} about the offline FSA algorithm, such as applications and theoretical analysis. 
For the square loss, the computation of the gradient 
\begin{equation*}
\frac{\partial L}{\partial \bbeta} = - \frac{\bX^T\mathbf{y}}{n} + \frac{\bX^T\bX\bbeta}{n}=\bS_{xx}\bbeta - \bS_{xy}
\end{equation*}
falls into the proposed running averages framework. 
It is not hard to verify that the OFSA is equivalent to the offline FSA.
The OFSA algorithm is presented in Algorithm \ref{alg:fsa}.
\begin{algorithm}[ht]
\caption{{\bf \small{Online FSA}}}\label{alg:fsa}
\begin{algorithmic}
\STATE {\bfseries Input:} RAVEs $\bS_{xx}, \bS_{xy}$, sample size $n$, true sparsity level $k^*$, and learning rate $\eta$.
\STATE {\bfseries Output:} Estimated coefficient vector $\bbeta$ with $\|\bbeta\|_0\leq k^*$.
\STATE Initialize $\bbeta = 0$.
\begin{itemize}
    \item[] For $t = 1$ to $T$:
    \begin{enumerate}
    \item Update $\bbeta \leftarrow \bbeta - \eta (\bS_{xx}\bbeta-\bS_{xy})$.
    \item Keep only $M_t$ variables with highest $\lvert\beta_j\rvert$.
    \item Renumber the $M_t$ features as $1,...,M_t$.
    \end{enumerate}
    \item[] End For
\end{itemize}
\STATE Fit the model on the selected features by OLS.
\end{algorithmic}
\end{algorithm}
%%%%%%%%%%%%%%%%%%%%%%%%%%%%%%%%%%%%%%%%%%%%%%%%%%%%%%%%%%%%%%%%%%%%%%%%%%%%%%%%%%%%%%%%%%%%%%%%%%
\subsection{Online Regularization Methods}
Penalized methods are proposed to select features, and can be mapped into the RAVEs framework.
They can estimate the parameter $\bbeta$ by solving the optimization problem (\ref{eq:penreg}), where $\lambda > 0$ is a hyper-parameter. The most popular method is the Lasso \citep{tibshirani1996regression}.
Moreover, the SCAD \citep{fan2001variable}, Elastic Net \citep{zou2005regularization}, and MCP \citep{zhang2010nearly} methods were proposed to deal with the variable selection and estimation problems. 
In this article, the gradient-based method with a thresholding operator $\mathbf{\Theta}(t;\lambda)$ is used to solve the regularized loss minimization problems \citep{she2009thresholding}. 
The general algorithm is presented in Algorithm \ref{alg:orm}. 
\begin{algorithm}[ht]
\caption{{\bf \small{Online Regularized Methods}}}\label{alg:orm}
\begin{algorithmic}
\STATE {\bfseries Input:} RAVEs $\bS_{xx}, \bS_{xy}$, sample size $n$, tuning parameter $\lambda$, and learning rate $\eta$.
\STATE {\bfseries Output:} Estimated coefficient vector $\bbeta$.
\STATE Initialize $\bbeta = 0$.
\begin{itemize}
    \item[] For $t = 1$ to $T$:
    \begin{enumerate}
    \item Update $\bbeta \leftarrow \bbeta - \eta (\bS_{xx}\bbeta-\bS_{xy})$.
    \item Update $\bbeta \leftarrow \mathbf{\Theta}(\bbeta; \eta\lambda)$.
    \end{enumerate}
    \item[] End For
\end{itemize}
\STATE Fit the model on the selected features by OLS.
\end{algorithmic}
\end{algorithm}
%%%%%%%%%%%%%%%%%%%%%%%%%%%%%%%%%%%%%%%%%%%%%%%%%%%%%%%%%%%%%%%%%%%%%%%%%%%%%%%%%%%%%%%%%%%%
\subsection{Online Classification Methods}
The proposed online learning methods not only select important features for the linear regression model but can also select variables for the binary classification model, even though these methods are based on the linear regression model with $\ell_2$ loss such as (\ref{eq:penreg}). 
In fact, for the binary classification problem with labels $+1$ and $-1$, the coefficient vector estimated from least squares is proportional to the coefficient vector by linear discriminant analysis without intercept \citep{friedman2001elements}.
This implies that the proposed online learning methods can effectively select variables for binary classification as well.
As presented in \cite{neykov2016l1}, under certain assumptions, one can use the penalized methods to select features for classification models such as the Logistic regression model.

As a special case of binary classification, the problem of class label imbalance is challenging in the area of machine learning, as methods tend to favor the majority class, which leads to biased performance. 
In computer vision, it means that the algorithm might struggle to accurately classify the minority class (positive samples) due to the lack of representative examples.
The RAVEs framework can handle this imbalanced data classification problem by computing separate RAVEs $\bS_{xx}^p, \bmu_{x}^p, n_p$ for the observations with positive labels, and the observations with negative labels, $\bS_{xx}^n, \bmu_{x}^n, n_n$. 
First, these RAVEs must be normalized as described in Section \ref{subsec2:Standardization} using the same mean and standard deviation. 
Therefore, the mean $\bmu_x^n$ and standard deviation $\bsigma_n = \sqrt{\text{diag}{\bS_{xx}^n} - (\bmu_{x}^n)^2}$ of the observation with negative labels can be used to standardize both RAVEs since they are usually the majority of the data.
In this case, a weighted $\ell_2$ loss function is appropriate:
\begin{equation*}
\begin{split}
L(\bbeta,\beta_0) 
=&\frac{w_p}{n_p}\sum_{i,y_i=1} (y_i-\bx_i^T\bbeta-\beta_0)^2 
+\frac{w_n}{n_n}\sum_{i,y_i=-1} (y_i-\bx_i^T\bbeta-\beta_0)^2 \\
=&\frac{w_p}{n_p} \|\bX_p\bbeta+\beta_0\bo_{n_p}-\by_p\|^2
+\frac{w_n}{n_n}\|\bX_n\bbeta+\beta_0\bo_{n_n}-\by_n\|^2,
\end{split}  
\end{equation*}
where each positive has weight $w_p/n_p$ and each negative has weight $w_n/n_n$.
In practice, we use $w_p=w_n=1$ as default values.
Observe that an intercept $\beta_0$ is used in this case to avoid standardizing the binary labels $\by$.
The gradient of the weighted loss will be
\begin{equation}\label{imbal:grad1}
\begin{split}
\frac{\partial L}{\partial \bbeta}(\bbeta,\beta_0)=&w_p (\frac{\bX_p^T\bX_p\bbeta}{n_p}+\frac{\bX_p^T\bo_{n_p}}{n_p}\beta_0- \frac{\bX_p^T\mathbf{y}_p}{n_p})
+w_n (\frac{\bX_n^T\bX_n\bbeta}{n_n}+\frac{\bX_n^T\bo_{n_n}}{n_n}\beta_0-\frac{\bX_p^T\mathbf{y}_n}{n_n}) \\
=&w_p (\bS_{xx}^p\bbeta+\bmu_{x}^p(\beta_0-1))+w_n (\bS_{xx}^n\bbeta+\bmu_{x}^n(\beta_0+1))\\
=&(w_p\bS_{xx}^p+w_n\bS_{xx}^n)\bbeta+(w_p\bmu_x^p+w_n\bmu_x^n)\beta_0+w_n\bmu_x^n-w_p\bmu_x^p, 
\end{split}   
\end{equation}
\begin{equation}\label{imbal:grad2}
\begin{split}
\frac{\partial L}{\partial \beta_0}(\bbeta,\beta_0) =&w_p (\frac{\bo_{n_p}^T\bo_{n_p}\beta_0}{n_p}+\frac{\bo_{n_p}^T\bX_p\bbeta}{n_p}- \frac{\bo_{n_p}^T\mathbf{y}_p}{n_p})
+w_n (\frac{\bo_{n_n}^T\bo_{n_n}\beta_0}{n_n}+\frac{\bo_{n_n}^T\bX_n\bbeta}{n_n}- \frac{\bo_{n_n}^T\mathbf{y}_n}{n_n})\\
=&w_p (\bmu_{x}^p\bbeta+\beta_0-1)+w_n (\bmu_{x}^n\bbeta+\beta_0+1)\\
=&(w_p\bmu_{x}^p+w_n\bmu_{x}^n)\bbeta+(w_p+w_n)\beta_0+w_n-w_p.
\end{split}   
\end{equation}
The proposed online methods in Algorithms \ref{alg:olsth} to \ref{alg:orm} can select the important variables for the imbalanced data classification by using the gradients (\ref{imbal:grad1}) and (\ref{imbal:grad2}).  
%%%%%%%%%%%%%%%%%%%%%%%%%%%%%%%%%%%%%%%%%%%%%%%%%%%%%%%%%%%%%%%%%%%%%%%%%%%%%%%%%%%%%%%%%%%% 
\subsection{\bf Memory and Computational Complexity}
In general, the memory complexity for the RAVEs is $\mathcal{O}(p^2)$ because $\bS_{xx}$ is a $p \times p$ matrix. 
The computational complexity of maintaining the RAVEs is $\mathcal{O}(np^2)$.
Except for OLSth, the computational complexity for obtaining the model using the running average-based algorithms is $\mathcal{O}(p^2)$ based on the limited number of iterations, each taking $\mathcal{O}(p^2)$ time.
As for OLSth, it is $\mathcal{O}(p^3)$ if done by Gaussian elimination or $\mathcal{O}(p^2)$ if done using an iterative method that takes much fewer iterations than $p$. 
We can conclude that the storage of RAVEs does not depend on the sample size $n$, and the computation is linear in $n$.
Hence, when $n >> p$, compared to the batch learning algorithms, the running averages-based methods need less memory and have less computational complexity. 
At the same time, they can achieve the same convergence rate as the batch learning algorithms.

%%%%%%%%%%%%%%%%%%%%%%%%%%%%%%%%%%%%%%%%%%%%%%%%%%%%%%%%%%%%%%%%%%%%%%%%%%%%%%%%%%%%%%%%%%%%%%%%%%%%%%%%%%%%%%%
\subsection{\bf Model Adaptation}
Detecting changes in the underlying model and rapidly adapting to the changes are common problems in online learning, and some applications are based on varying-coefficient models \citep{javanmard2017perishability}. 
Our RAVEs online methods can adapt to coefficient changes for large-scale streaming data. For that, the update equation \eqref{eq:raveupdate} can be regarded in a more general form as
\begin{equation}\label{adapt:eq}
\bmu_x^{(n+1)}=(1 - \alpha_n)\bmu_x^{(n)}+\alpha_n\bx_{n+1},
\end{equation}
where we only show one of the RAVEs for illustration but the same type of updates are used for all of them.

The original RAVEs use $\alpha_n=1/(n+1)$, which gives all observations equal weight in the RAVEs. 
For the coefficient-varying models, we use a larger value of $\alpha_n$, which gives more weight to the recent observations. 
However, too much adaptation is also not desirable because in that case, the model will not be able to recover some weak coefficients that can only be recovered given sufficiently many observations.
More details about simulations and applications will be covered in Section \ref{sec5:simandrealdata}.

%%%%%%%%%%%%%%%%%%%%%%%%%%%%%%%%%%%%%%%%%%%%%%%%%%%%%%%%%%%%%%%%%%%%%%%%%%%%%%%%%%%%%%%%%%%%%%%%%%%%%%%%%%%%%%%%%%%%%%%%%%%%%%%%%%%%%%
\section{Theoretical Analysis}\label{sec4:theory}
The theoretical justification for the proposed methods under the case of the linear regression model is provided in this section.  
First, by using the Proposition \ref{prop1}, we have the equivalence of the online penalized methods including Lasso, Elastic Net, and MCP with their offline counterparts, and thus all their theoretical guarantees of convergence, consistency, oracle inequalities, etc., can carry over to their online counterparts \citep{wainwright2009sharp, loh2017support}.
Then, theorems for variable selection consistency are provided for OLSth and OFSA methods in the low-dimensional case.
%When the sample size $n$ is large, the OLSth and OFSA methods can select important variables with a large probability.
Finally, an upper bound of the regret for the OLSth method is provided by using the technique in \cite{yuan2014gradient}. 
%The main idea of this proof is inspired by \cite{yuan2014gradient}.
Additionally, since data normalization is considered in our theoretical analysis, the intercept $\beta_0$ is not considered.  
All the proofs are provided in the supplementary material which is available online.

\begin{theorem}\label{them1}{\bf (True Support Recovery for OLSth)}
\textcolor{black}{Suppose that we have a sparse linear regression model
\begin{equation*}
y_i = \bx^T_i\bbeta^* + \epsilon_i,    
\end{equation*}
where $\bx_i \in \mathbb{R}^p$, $i = 1,2,\cdots$, is a random vector independently drawn from the multi-variate Gaussian distribution $N(0, \bSigma)$, and the random noise $\epsilon_i$ is sub-Gaussian with the parameter $\sigma$.
Let $S_{\bbeta^*}=\{j, \beta_j^*\not =0\}$ and $|S_{\bbeta^*}|=k^*$.
Assume that $p/n \to 0$ when $n \to +\infty$.}

\textcolor{black}{a) If
\begin{equation*}
    \min_{j \in S_{\bbeta^*}}|\beta_{j}^*|> 4\sigma\sqrt{\frac{\log(p)}{n\lambda}},
    \text{ where } 0 < \lambda < \lambda_{\min} (\bSigma), 
\end{equation*}
then with probability at least $1 - \exp(-p)-2\exp(-\log(p))$, the index set of the top $k^*$ values of the vector $|{\hat\bbeta}|$ is exactly $S_{\bbeta^*}$, where $\hat \bbeta$ is the OLS estimator.}

\textcolor{black}{b) If
\begin{equation*}
    \min_{j \in S_{\bbeta^*}}|\beta_{j}^*|> 4\sigma\sqrt{\frac{p}{n\lambda}},
    \text{ where } 0 < \lambda < \lambda_{\min} (\bSigma), 
\end{equation*}
then with probability at least $1 - \exp(-p)-2p\exp(-2p)$, the index set of the top $k^*$ values of the vector $|{\hat\bbeta}|$ is exactly $S_{\bbeta^*}$, where $\hat \bbeta$ is the OLS estimator.}
\end{theorem}

\textcolor{black}{
Theorem \ref{them1} shows a theoretical guarantee of true support recovery for OLSth.
We can conclude that the probability of true feature recovery does not depend on the true sparsity level $k^*$.
If the number of predictors $p$ increases, then with a very high probability, the support of true features can be recovered. If the number of predictors $p$ is fixed or increasing slowly and the sample size $n$ is increasing dramatically, the proposed OLSth method can recover very weak signals.
It is worth noting that we can relax the assumption that when $n \to +\infty$, then $p/n \to 0$ to the $p = \mathcal{O}(\exp(n^{\alpha}))$ and $n > p$, where $\alpha \in \mathbb{R}$ is a small value such as $\alpha \in (0, 1/10)$. Then, suppose that the minimum $\bbeta$ condition
\begin{equation*}
    \min_{j \in S_{\bbeta^*}}|\beta_{j}^*|> 4\sigma\sqrt{\frac{\log(p)}{n^{\alpha}\lambda}},
    \text{ where } 0 < \lambda < \lambda_{\min} (\bSigma) 
\end{equation*}
holds, with the probability at least $1 - \exp(-p) - 2p^{1 - n^{1-\alpha}}$, the index set $S_{\bbeta^*} = S_{\hat\bbeta}$.} 
Moreover, the theoretical justification for the data standardization case is presented.
\begin{remark}\label{remark2}
Denote $\bPi = \text{diag}\{\sigma_{\bx_1}, \sigma_{\bx_2}, \cdots,  \sigma_{\bx_p}\}^{-1}$, and $\hat{\bPi} = \text{diag}\{\hat{\sigma}_{\bx_1}, \hat{\sigma}_{\bx_2}, \cdots,  \hat{\sigma}_{\bx_p}\}^{-1}$. 
Given the conditions $\min_{j \in S_{\bbeta^*}}|\sigma_{\bx_j}\beta_{j}^*| > 4\sigma\sqrt{\frac{\log(p)}{\lambda n}}$, for $\lambda$ satisfying $0 < \lambda \leq \lambda_{\min} (\bSigma)$, then with high probability the index set of the top $k^*$ values of $|{\hat \beta_j}|$ is exactly $S_{\bbeta^*}$, where $\hat\bbeta=(\mathbf{{\tilde X}}^T\mathbf{{\tilde X}})^{-1}\mathbf{{\tilde X}}^T\mathbf{y}$ is the OLS estimate with standardized data matrix $\mathbf{{\tilde X}} = \bX\hat{\bPi}$. 
\end{remark}

% \begin{theorem}\label{them1}{\bf (True Support Recovery for OLSth)}
% With the same notations as Proposition \ref{prop2}, if
% \begin{equation}
%     \min_{j \in S_{\bbeta^*}}|\beta_{j}^*| > \frac{4\sigma} {\sqrt{\lambda}}\sqrt{\frac{\log(p)}{n}},
%     \text{ for } \lambda =0.9 \lambda_{\min}(\sqrt{\mathbf{\Sigma}}) - \rho(\mathbf{\Sigma})\sqrt{\frac{p}{n}}, \label{eq:bndthm1}
% \end{equation}
% where $\rho(\mathbf{\Sigma})$ is the largest diagonal value of $\Sigma$, then with probability $1- 2p^{1-2n^{1-\alpha}} - e^{-n/200}$,
% the index set of top $k^*$ values of $|{\hat{\beta}_j}|$, $j = 1, 2, \cdots, p$, is exactly $S_{\bbeta^*}$, where $\hat{\bbeta}$ is OLS estimator.
% \end{theorem}

Then, we will consider the theoretical guarantees of true support recovery for the OFSA method.
First, the definitions for restricted strong convexity$\backslash$smoothness are introduced.

\begin{definition}{\bf (Restricted Strong Convexity$\backslash$Smoothness)}
For any integer $s > 0$, a differentiable function $f(x)$ is called restricted strongly convex (RSC) with parameter $m_s$ and restricted strongly smooth (RSS) with parameter $M_s$ if there exist $m_s, M_s >0$ such that
\begin{equation*}
\begin{split}
\frac{m_s}{2}\|\bbeta - \bbeta'\|^2 &\leq f(\bbeta) - f(\bbeta') - \langle \nabla f(\bbeta'), \bbeta - \bbeta' \rangle \\ 
&\leq \frac{M_s}{2}\|\bbeta - \bbeta'\|^2,
\ \forall\|\bbeta - \bbeta'\|_0 \leq s. \\ 
\end{split}
\end{equation*} 
\end{definition}
It is worth noting that in the low-dimensional linear regression, the RSC$\backslash$RSS conditions hold with at least a probability $1 - 2\exp(-p)$ if there exist constants $m_s$ and $M_s$ such that
$$
0 < m_s < \lambda_{\min}(\bSigma) < \lambda_{\max}(\bSigma) < M_s < +\infty.
$$
Therefore, the following condition 
\begin{equation*}
m_s\|\bbeta - \bbeta'\|^2 \leq \frac{1}{n}\|\bX(\bbeta - \bbeta')\|^2 \leq M_s\|\bbeta - \bbeta'\|^2, \forall\|\bbeta - \bbeta'\|_0 \leq s,
\end{equation*}
holds, which is the restricted isometric property (RIP) condition in linear regression. 
\begin{proposition}\label{prop2}
Consider a sparse linear regression model
$$
y_i = \bx^T_i\bbeta^* + \epsilon_i,
$$
where $\bx_i \in \mathbb{R}^p$, $i = 1,2,\cdots$, is a random vector independently drawn from the multi-variate Gaussian distribution $N(0, \bSigma)$, and the random noise $\epsilon_i$ is sub-Gaussian with the parameter $\sigma$. Let $0 < m_s < \lambda_{\min}(\bSigma) < \lambda_{\max}(\bSigma) < M_s < +\infty$.
Let $\bbeta^*$ be a $k^*$-sparse vector and $S_{\bbeta^*}=\{j, \beta_j^*\not =0\}$, hence $\|\bbeta^*\|_0 = k^*$. Let $\bbeta^{(t)}$ be the OFSA coefficient vector at iteration $t$, $S_{\bbeta^{(t)}}$ be its support, $k= \mid S_{\bbeta^{(t)}}\mid\geq k^*$ and $s = k + k^*$. 
If $f$ is a differentiable function which is $m_s$-convex and $M_s$-smooth, then for any learning rate $0<\eta < 2m_s/M^2_s$, with probability  at least $1 - 2\exp(-p)$, we have
\begin{equation*}
\|\bbeta^{(t+1)} - \bbeta\| \leq \varphi\rho\|\bbeta^{(t)} - \bbeta\| + \varphi\eta\sqrt{s}\|\nabla f(\bbeta)\|_{\infty},
\end{equation*} 
where $\varphi = (\sqrt{5}+1)/2$ and 
\begin{equation*}
\rho = \sqrt{1 - 2\eta m_s + \eta^2M_s^2 } < 1.
\end{equation*}
\end{proposition}

\begin{theorem}{\bf (Convergence of OFSA)}\label{thm2:ofsa}
With the same assumptions as Proposition \ref{prop2}, let $\bbeta^{(0)} = 0$ and $S_{\bbeta^{(0)}} = \{1, 2, \cdots, p\}$.
Assume we have $M_s/m_s < 1.26$ for any $k^* \leq s \leq p$. 
Let $\bPi = \text{diag}\{\sigma_{\bx_1}, \sigma_{\bx_2}, \cdots,  \sigma_{\bx_p}\}^{-1}$ be the diagonal matrix with the inverse of true standard deviations of random variables $X_1,..., X_p$ respectively. 
Then, with probability $1 - 4p^{-1} - 2\exp(-p)$, the OFSA coefficient vector $\bbeta^{(t)}$ satisfies
\begin{equation*}
\|\bbeta^{(t)} - \bPi^{-1}\bbeta^*\|\leq (\varphi\rho)^t\| \bPi^{-1}\bbeta^*\| 
+ 2\varphi\eta\frac{\sigma + 2\|\bPi^{-1}\bbeta^*\|_\infty}{1 - \varphi\rho}\sqrt{\frac{p\log p}{n}}.
\end{equation*} 
\end{theorem}

Please note that the dimension of the vector $\bbeta^{(t)}$ will reduce from $p$ to $k^*$, thus we apply Proposition \ref{prop2} recursively with varying $k \geq k^*$. Here, we assume that $\|\bbeta^{(t)}\|_0 = k^*$. 
Now we show that the OFSA algorithm can recover the support of true features with high probability.
\begin{corollary}{\bf (True Support Recovery for OFSA)}\label{col1}
Under the conditions of Theorem \ref{thm2:ofsa}, let 
\begin{equation*}
\beta_{\min}:=\min_{j \in S_{\bbeta^*}}|\beta_j| > \frac{4\eta(\sigma+2\|\bPi^{-1}\bbeta^*\|_\infty)}{1 - \varphi\rho}\sqrt{\frac{p\log(p)}{n}}.
\end{equation*} 
Then after $t = [\frac{1}{\varphi\rho} \log(\frac{10\|\bPi^{-1}\bbeta^*\|}{\beta_{\min}})] + 1$ iterations, the OFSA algorithm will output $\bbeta^{(t)}$ satisfying $S_{\bbeta^*} = S_{\bbeta^{(t)}}$ with probability $1 - 4p^{-1} - 2\exp(-p)$.
\end{corollary}

Then we consider regret bounds for the OLS and OLSth algorithms. All the feature selection algorithms we mentioned will degenerate to OLS if the true features are selected. First, we define the regret for a sparse model with sparsity levels $\|\bbeta\|_0 \leq k^*$:
\begin{equation}\label{eq:regret}
R_n = \frac{1}{n}\sum_{i=1}^{n} f(\bbeta_i; \bz_i) - \min_{\bbeta \in \mathbb{R}^p , \|\bbeta\|_0 \leq k^*}\frac{1}{n}\sum_{i=1}^{n}f(\bbeta; \bz_i),
\end{equation}
where $\bbeta_{i+1}$ is the estimated coefficient vector based on the observation $\bz_i = (\bx_i, y_i)$ and the $\bbeta_{i}$. We define the initial $\bbeta_1 = 0$. 
Then we assume that the loss functions $f$ from \eqref{eq:regret} are twice continuously differentiable and for all $\bbeta \in \mathbb{R}^p$, there is existing a constant $G$, satisfying $\|\nabla f(\bbeta)\|\leq G$. 
%We denote  $ \bbeta_{n+1} = \arg\min_{\bbeta}\frac{1}{n}\sum_{i=1}^{n}f(\bbeta)$ and $(\bX^T\bX)_n=\sum_{i=1}^{n}\bx_i\bx_i^T$. 
% We will need the following assumptions:
% \begin{assumption}\label{assump1}
% Given $n > p$, there is $0 < m < M$ satisfying
% \[
% 0 < m < \lambda_{\min}(\frac{1}{n}(\bX^T\bX)_n) < \lambda_{\max}(\frac{1}{n}(\bX^T\bX)_n) < M.
% \] 
% \end{assumption}
% \begin{assumption}\label{assump2}
% Given $n > p$, there exist constants $D$ and $G$ such that $\|\bbeta_i - \bbeta_j\| < D, \forall i, j > n$ and $\|\nabla f(\bbeta_i)\|\leq G, \forall i\geq n$.
% \end{assumption}
\begin{proposition}{\bf (Regret of OLS)}\label{prop3}
Consider a linear regression model
$$
y_i = \bx^T_i\bbeta^* + \epsilon_i,
$$
where $\bx_i \in \mathbb{R}^p$, $i = 1,2,\cdots$, is a random vector independently drawn from the multi-variate Gaussian distribution $N(0, \bSigma)$, the random noise $\epsilon_i$ is sub-Gaussian with the parameter $\sigma$, and $\bbeta^*$ is the true coefficient vector. 
Denote $\bbeta_i$, $i = 1, 2, \cdots$, as the estimator for the online least square method by using the observations $\bx_1, \bx_2, \cdots$, and the initial $\bbeta_1 = 0$.
When $n > p$ and assuming that $\|\nabla f(\bbeta)\|\leq G$, where $G$ is a constant and $\bbeta \in \mathbb{R}^p$, with probability $1 - \exp\{-c_1p\}$, the regret for the online least square method can be upper bounded as
 \begin{equation*}
\text{R}_n=\frac{1}{n}\sum_{i=1}^{n} (y_i - \bx_i^T\bbeta_i)^2 - 
\min_{\bbeta}\frac{1}{n}\sum_{i=1}^{n}(y_i - \bx_i^T\bbeta)^2\leq\mathcal{O}(\frac{\log(n)}{n}).
\end{equation*}
\end{proposition}
\begin{theorem}{\bf (Regret of OLSth)}\label{thm3:regret}
Consider a sparse linear regression model
$$
y_i = \bx^T_i\bbeta^* + \epsilon_i,
$$
where $\bx_i \in \mathbb{R}^p$, $i = 1,2,\cdots$, is a random vector independently drawn from the multi-variate Gaussian distribution $N(0, \bSigma)$, the random noise $\epsilon_i$ is sub-Gaussian with the parameter $\sigma$, and $\bbeta^*$ is the true coefficient vector.
Let $S_{\bbeta^*}=\{j, \beta_j^*\not =0\}$, $|S_{\bbeta^*}|=k^*$ and
\begin{equation*}
    \min_{j \in S_{\bbeta^*}}|\beta_{j}^*|> 4\sigma\sqrt{\frac{\log(p)}{n_0\lambda}},
    \text{ where } 0 < \lambda < \lambda_{\min}(\bSigma). %\label{eqn:thm1bound}
\end{equation*}
When $n_0 >> p$, with the probability at least $1 - \exp(-\log(p)) - \exp(-p)$,
the regret of OLSth satisfies:
\begin{equation*}
\begin{split}
    \text{R}_n = \frac{1}{n}\sum_{i=1}^{n} (y_i-\bx_i^T\bbeta_i)^2  -  
\min_{\|\bbeta\|_0 \leq k}\frac{1}{n}\sum_{i=1}^{n}(y_i-\bx_i^T\bbeta)^2
\leq \mathcal{O}(\frac{\log^2(n)}{n}),
\end{split}
\end{equation*}
where $n_0 = \mathcal{O}(\log^2(n))$.
\end{theorem}

%%%%%%%%%%%%%%%%%%%%%%%%%%%%%%%%%%%%%%%%%%%%%%%%%%%%%%%%%%%%%%%%%%%%%%%%%%%%%%%%%%%%%%%%%%%%%%%%%%%%%%%%%
%%%%%%%%%%%%%%%%%%%%%%%%%%%%%%%%%%%%%%%%%%%%%%%%%%%%%%%%%%%%%%%%%%%%%%%%%%%%%%%%%%%%%%%%%%%%%%%%%%%%%%%%%
\section{Experiments}\label{sec5:simandrealdata}
This section shows an evaluation of the proposed algorithms and a comparison with offline and online learning methods. 
First, numerical experiments on synthetic data are presented, comparing the performance of feature selection and prediction. 
Then, regret plots are also provided for the running averages-based methods and compared with classical online methods. 
Finally, an evaluation of real datasets is presented, for both regression and classification.
All simulation experiments are run on a desktop computer with a Core i5 - 4460S CPU and 16Gb memory.

\subsection{Experiments for Simulated Data}\label{subsec:simdata}
The simulated data is generated by uniformly correlated predictors: given a scalar $\alpha$, a $z_i \sim \mathcal{N}(0, 1)$ is generated for each observation and then the observation is obtained as
\begin{equation*}
\bx_i = \alpha z_i \mathbf{1}_{p \times 1 } + \bu_i, \text{ with } \bu_i \sim \mathcal{N}(0, \mathbf{I}_p).
\end{equation*}
The observations $\bx_1, \bx_2, ...$ are generated and let 
$\bX = (\bx^T_1, \bx^T_2, \cdots, \bx^T_n)^T$ be the data matrix. 
It is not hard to verify that the correlation between any pair of predictors is $\alpha^2/(1+\alpha^2)$. 
We set $\alpha = 1$ in our experiments, hence the correlation between any two variables is 0.5. 
Given $\bX$, the dependent response vectors $\mathbf{y}$ are generated from the following linear models, for regression and classification, respectively, 
\begin{eqnarray*}
&\mathbf{y} = \bX\bbeta^* + \bet, \text{ with } \bet \sim \mathcal{N}(0,\mathbf{I}_n), \\
&\mathbf{y} = \text{sign}(\bX\bbeta^* + \bet), \text{ with } \bet \sim \mathcal{N}(0,\mathbf{I}_n),
\end{eqnarray*}
where $\bbeta^*$ is a $p$-dimensional sparse parameter vector. 
The true coefficients $\beta^*_j = 0$ except $\beta^*_{10j^*} = \beta$, $j^* = 1, 2, \cdots, k$, where $\beta$ is a signal strength parameter. 
We can observe that the classification data cannot be perfectly separated by a linear model.

%%%%%%%%%%%%%%%%%%%%%%%%%%%%%%%%%%%%%%%%%%%%%%%%%%%%%%%%%%%%%%%%%%%%%%%%%%%
The simulation is based on the following data parameter setting: $p = 1000$ and $k = 100$, as well as $p\in \{100, 1000, 10000\}$ with $k=50$. 
Two signal strengths $\beta \in \{0.01,1\}$ (weak and strong signal) are considered. 
The sample size $n$ varies from $300$ to $10^6$ in the simulation. 
For regression, the proposed methods are compared with SADMM \citep{ouyang2013stochastic}, the offline Lasso \citep{tibshirani1996regression} and the truncated stochastic gradient descent (TSGD) \citep{langford2009sparse}:
\begin{equation*}
{\tilde \bbeta^{(n+1)}} 
= \text{Truncate} \big( \bbeta^{(n)} + \eta(y_n - \bx_n^T\bbeta^{(n)})\bx_n, \eta\lambda, \lambda \big),
\end{equation*}
where the Truncate operator is the $T_1$ operator from \cite{langford2009sparse}. 
\begin{figure}[ht]
\centering
\begin{tabular}{ccc}
\includegraphics[width=0.3\linewidth]{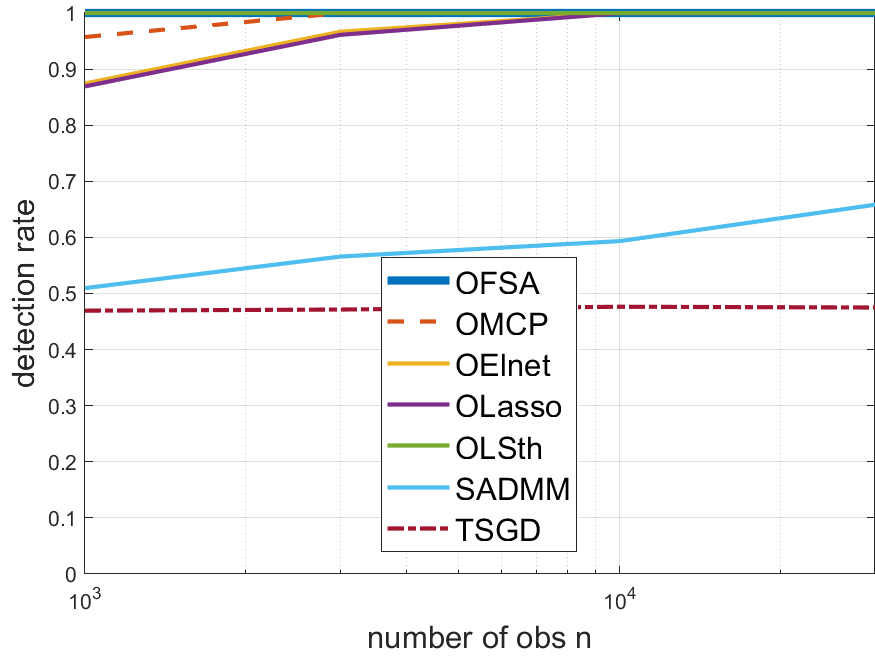}
&\includegraphics[width=0.3\linewidth]{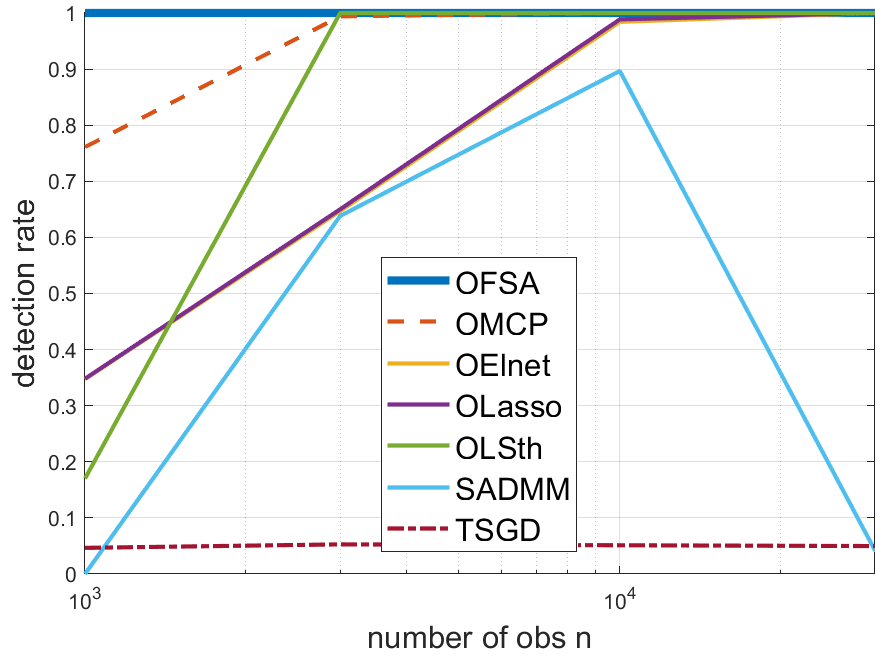}
&\includegraphics[width=0.3\linewidth]{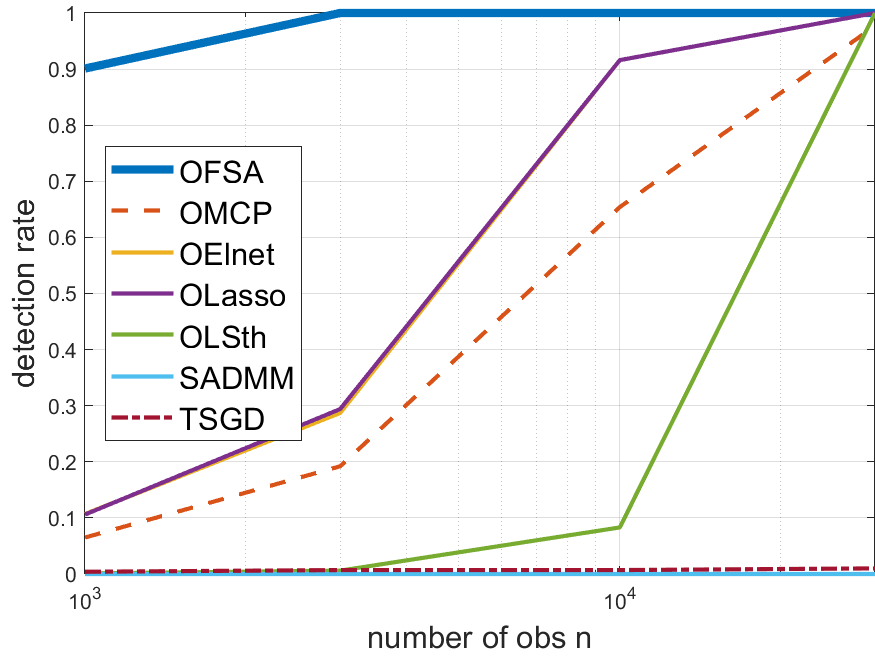}\\
$p=100, \beta=1$ &$p=1000, \beta=1$ &$p=10000, \beta=1$\\
\includegraphics[width=0.3\linewidth]{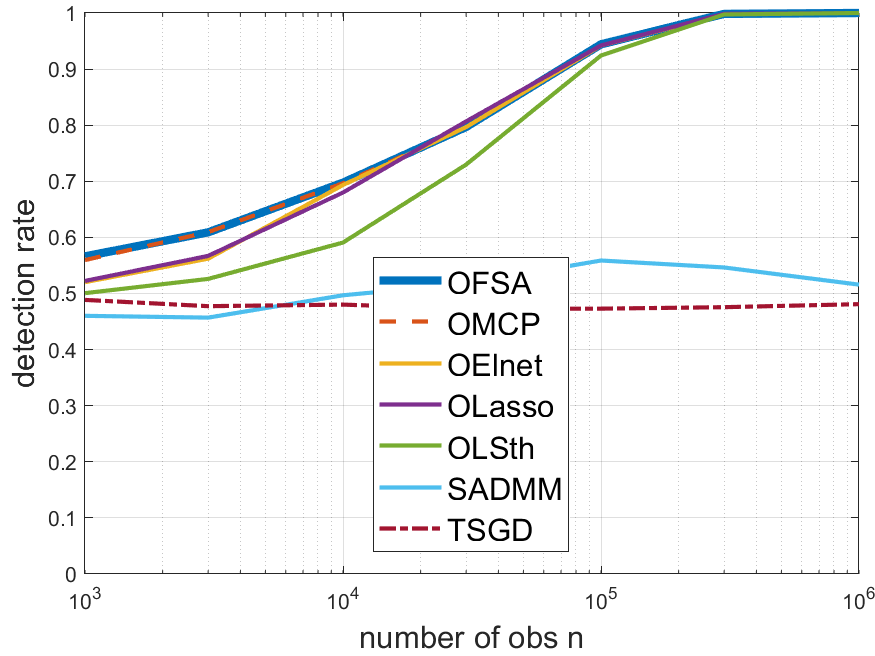}
&\includegraphics[width=0.3\linewidth]{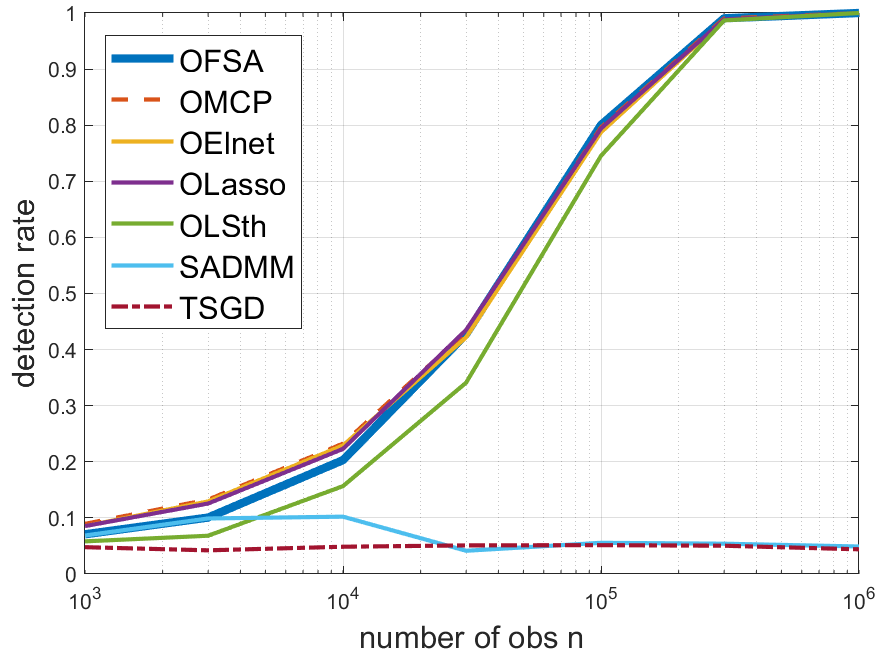}
&\includegraphics[width=0.3\linewidth]{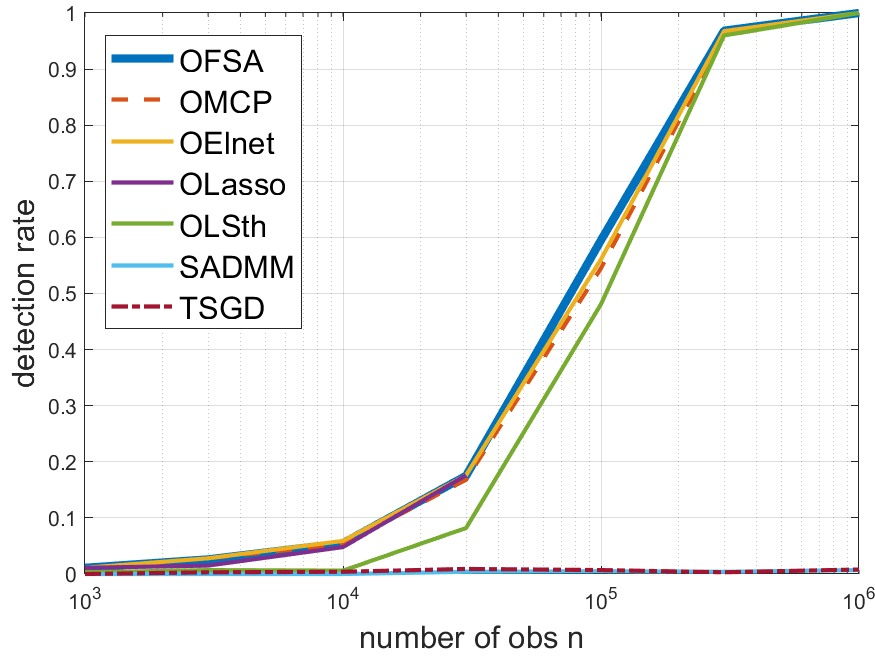}\\
$p=100, \beta=0.01$ &$p=1000, \beta=0.01$  &$p=10000, \beta=0.01$\\
\end{tabular}
\caption{Plot of the variable detection rate (DR) vs $n$ in regression, for $k=50$ true variables. Top: strong signal $\beta=1$, bottom: weak signal $\beta=0.01$.} \label{fig:plot_dr_n}
\end{figure}

For classification, four existing methods are used for comparison: the OPG \citep{duchi2009efficient}  and RDA \citep{xiao2010dual} methods for Elastic Net, the first order online feature selection (FOFS) method \citep{wu2017large} and the second order online feature selection (SOFS) method \citep{wu2017large}.

\setlength{\tabcolsep}{1.5pt}
\begin{table}[t]
\centering
\caption{\small{Comparison between running average-based methods and other online or offline methods for regression, averaged over 100 runs.}}\label{tab:DR_RMSE_small}
\scalebox{0.52}{
\begin{tabular}{c|ccccccc|ccccccc|cccccccc}
&\multicolumn{7}{c}{Variable Detection Rate (\%)}&\multicolumn{7}{|c}{RMSE} &\multicolumn{8}{|c}{Time (s)} \\
\hline
$n$  &Lasso&TSGD &SADMM &OLSth &OFSA &OMCP &OElnet    &Lasso&TSGD &SADMM &OLSth &OFSA &OMCP&OElnet    &Lasso&TSGD &SADMM &OLSth &OFSA &OMCP &OElnet&RAVE \\
\hline
\multicolumn{23}{c}{$p=1000, k=100$, strong signal $\beta=1$}  \\
\hline
$300$          &24.68  &10.88  &11.99    &64.56 &{\bf 71.09} &41.24  &24.68     &14.15    &98.56  &99.54   &8.641 &{\bf 7.605}&12.40   &14.16       
               &3.17   &4.29   &1.09     &0.050 &0.289       &15.1   &13.7      &0.012    \\
$500$          &27.72  &10.45  &13.14    &85.09 &{\bf 94.05} &53.50  &27.68     &12.89    &99.47  &97.85   &4.758 &{\bf 2.989}&9.786   &12.91       
               &3.42   &7.05   &1.78     &0.051 &0.288       &15.4   &10.6      &0.015    \\
$10^3$         &32.14  &10.15  &18.10    &94.53 &{\bf 99.81} &73.71  &32.12     &11.63    &100.11 &95.05   &2.657 &{\bf 1.136}  &6.282 &11.61 
               &4.33   &14.0   &5.33     &0.052 &0.289       &15.5   &9.65      &0.026 \\
$3\cdot10^3$  &46.05   &10.50    &41.23    &{\bf 100}  &{\bf 100}  &98.02     &45.19     &9.464  &100.23 &93.50   &{\bf 1.017}  &{\bf 1.017}  &1.745           &9.557 &26.9            &42.9    &15.7   &0.051    &0.288      &13.9       &7.11     &0.076\\   
$10^4$            &72.40   &10.20   &65.78     &{\bf 100}  &{\bf 100}  &{\bf 100}&72.42      &6.07      &99.89    &94.92   &{\bf 1.003}  &{\bf 1.003 } &{\bf 1.003}    &6.042 
&47.3            &140 &51.8   &0.051     &0.288      &6.51       &5.89     &0.246\\
\hline
\multicolumn{23}{c}{$p=1000, k=100$, weak signal $\beta=0.01$} \\
\hline
$500$     &11.77  &7.63   &11.47    &10.66   &11.18     &12.82          &11.70         &1.219    &1.417    &1.385   &1.270  &1.248  &1.168    &1.217               &3.97   &7.17   &1.809    &0.051   &0.289     &13.1           &12.87         &0.016\\
$10^3$    &14.09  &6.60   &13.53    &10.53   &12.40     &{\bf 15.55} &14.08    &1.128    &1.416    &1.363   &1.075  &1.169  &{\bf 1.049}&1.124                     &5.35   &14.2   &6.70    &0.052    &0.288      &13.2        &9.74     &0.026\\
$10^4$         &31.58  &7.28   &19.80    &22.48 &{\bf 32.47}&32.32       &31.54     &1.009    &1.413    &1.370   &1.025   &1.006  &{\bf 1.005}  &1.006
&48.1            &141       &67.8  &0.051 &0.287       &15.0        &4.96        &0.249\\
$10^5$            &81.93        &7.70   &11.30 &80.55       &{\bf 85.14}&84.86        &81.80     &{\bf 1.001}&1.414        &1.382 &1.003        &1.003       &1.003          &1.003
&452            &1415         &680  &0.051 &0.287       &15.9       &5.12        &2.46\\
$3\cdot 10^5$&98.66         &7.17   &10.80 &98.94       &{\bf 99.27}&99.26       &98.71     &0.999        &1.412         &1.383 &{\bf 0.998} &{\bf 0.998}&{\bf 0.998} &{\bf 0.998}
&1172             &4205         &2044   &0.051 &0.287       &14.0        &3.75       &7.33\\
$10^6$           &-                &6.75    &11.14 &{\bf 100}   &{\bf 100}   &{\bf 100}   &{\bf 100} &-              &1.413         &1.388 &{\bf 0.996} &{\bf 0.996}&{\bf 0.996} &{\bf 0.996}
&-                   &13654       &6427   &0.051 &0.288        &7.35       &1.73       &24.4\\
\hline
\end{tabular}}
\end{table}
%%%%%%%%%%%%%%%%%%%%%%%%%%%%%%%%%%%%%%%%%%%%%%%%%%%%%%%%%%%%%%%%%%%%
\begin{table}[t]
\centering
\caption{\small{Comparison between running average-based methods and the other online methods for classification, averaged 100 runs.}}\label{table:DR&AUC}
\scalebox{0.51}{
\begin{tabular}{c|cccccccc|cccccccc|ccccccccc}
&\multicolumn{8}{c}{Variable Detection Rate (\%)}&\multicolumn{8}{|c}{AUC}&\multicolumn{9}{|c}{Time (s)}\\
\hline
n &FOFS&SOFS&OPG&RDA&OFSA&OLSth&OLasso&OMCP  &FOFS&SOFS&OPG&RDA&OFSA&OLSth &OLasso&OMCP&FOFS&SOFS&OPG&RDA&OFSA&OLSth&OLasso&OMCP&RAVE\\
\hline
\multicolumn{25}{c}{$p=1000, k=100$, strong signal $\beta=1$}\\
\hline
$10^3$         &10.64&10.06 &9.72 &12.66&14.44&9.90 &16.58&{\bf 18.57}&{\bf 0.994}&0.992 &0.992&0.991&0.973&0.975&0.986&0.992
               &0.001&0.001 &0.042&0.051&0.005&0.001&0.081&0.159      &0.026\\
$10^4$         &10.64&10.19 &10.46&11.91&38.89&30.30&34.70&{\bf 41.54}&0.995&0.992 &0.992&0.991&0.995&0.990&{\bf 0.996}&{\bf 0.996}
               &0.001&0.001 &0.490&0.848&0.005&0.001&0.080&0.160      &0.247\\
$3\times10^4$  &10.64&9.95  &10.42&10.34&{\bf 67.67}&59.32&56.18&67.52&0.994&0.992 &0.992&0.989&{\bf 0.998}&0.996&0.997&{\bf 0.998}
               &0.003&0.004 &1.47 &2.21 &0.005&0.001&0.083&0.158&0.742\\
$10^5$         &10.64&9.95  &10.43&11.08&{\bf 94.95}&93.21&86.90&94.77  &0.994&0.992 &0.992&0.990&{\bf 1.000}&{\bf 1.000}&0.999&{\bf 1.000}
               &0.010&0.015&4.90&6.12&0.005&0.001&0.079&0.159&2.48\\        
\hline
\multicolumn{25}{c}{$p=1000, k=100$, weak signal $\beta=0.01$}\\
\hline
$10^3$         &10.97&10.06 &9.75 &11.19&10.57&9.98 &{\bf 12.94}   &12.47   &{\bf 0.832}&0.831  &0.831&0.830&0.783&0.797&0.823&0.773
               &0.001&0.001 &0.042&0.048&0.005&0.001&0.080&0.156   &0.028\\
$10^4$         &13.40&10.19 &10.00&10.37&19.41&15.93&22.55&{\bf 23.81} &0.827&0.829  &0.828&0.828&0.824&0.815&0.829&{\bf 0.830}                      &0.001&0.001 &0.494 &0.815 &0.005&0.001&0.073&0.148&0.249\\    
$3\times10^4$  &15.86&9.95  &10.23 &10.34&34.46&27.35&35.14&{\bf 37.70} &0.827&0.829  &0.829&0.829&0.831&0.827&{\bf 0.832}&{\bf 0.832}                 &0.003&0.004 &1.48 &2.09 &0.005&0.001&0.074&0.152 &0.743\\
$10^5$         &17.36 &9.95  &10.32 &10.91&64.84&56.42&61.07&{\bf 64.95} &0.830&0.831  &0.831 &0.830&{\bf 0.834}&0.833&{\bf 0.834}&{\bf 0.834}         &0.010&0.015 &4.94 &5.83 &0.005&0.001&0.078&0.161&2.47\\
$3\times10^5$&17.13 &9.23  &10.32 &10.37&91.55&88.91&88.69&{\bf 91.58} &0.826&0.828  &0.828 &0.827 &{\bf 0.833}&{\bf 0.833}&{\bf 0.833}&{\bf 0.833} 
&0.030&0.044 &14.8 &17.3 &0.005&0.001&0.073&0.164&7.45\\
$10^6$       &17.72 &9.91  &-         &-        &{\bf 99.97}&99.94&99.88&{\bf 99.97} &0.828&0.829  &-         &-         &{\bf 0.834}&{\bf 0.834}&{\bf 0.834}&{\bf 0.834}  &0.100&0.146 &-         &-         &0.005&0.001&0.039&0.110&24.9\\          
\hline    
\end{tabular}}
\end{table} 
For each method, the sparsity controlling parameter was tuned to obtain $k$ variables. This can be done directly for OFSA and OLSth, and indirectly through the penalty parameters for the other methods. 
In RDA, OPG, and SADMM, 200 values of $\lambda$ were used on an exponential grid. The final $\lambda$ was chosen to induce $\hat{k}$ non-zero features, where $\hat{k}$ is the largest number of non-zeros features smaller than or equal to $k$, the number of true features.

The following criteria are used for evaluation: the true variable detection rate (DR), the root mean square error (RMSE) on the test data for regression, the area under the ROC curve (AUC) on the test data for classification, and the running time (Time) of the algorithms. 
The variable detection rate DR is defined as the average number of true variables correctly detected by an algorithm divided by the number of true variables. So if $S_\bbeta$ is the set of detected variables and $S_{\bbeta^*}$ are the true variables, then
\begin{equation*}
DR=\frac{E(\lvert S_\bbeta\cap S_{\bbeta^*}\lvert)}{\lvert S_{\bbeta^*} \lvert}.
\end{equation*}

\begin{figure}[ht]
\centering
\begin{tabular}{ccc}
\includegraphics[width=0.3\linewidth]{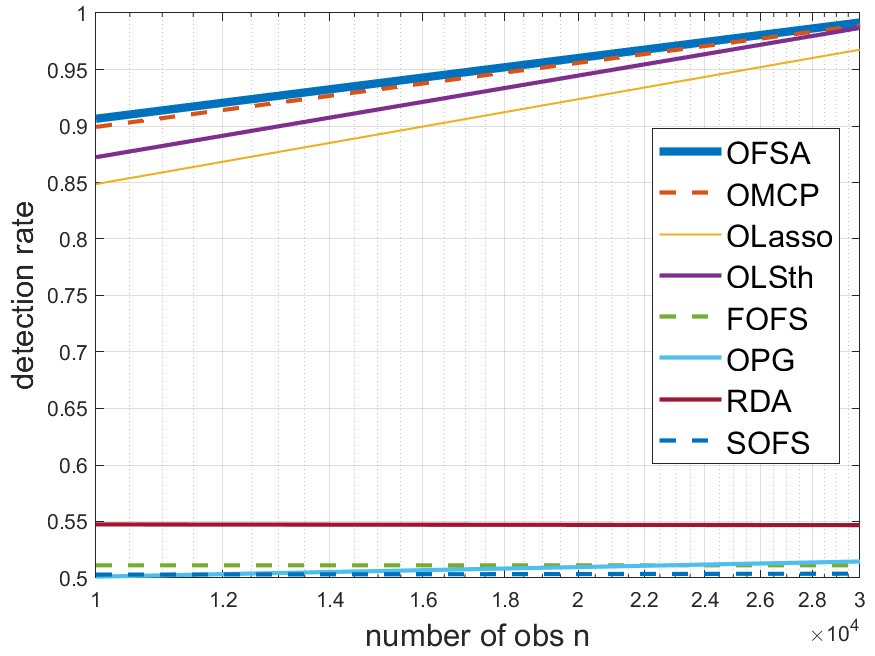}
\hspace{-2mm}&\includegraphics[width=0.3\linewidth]{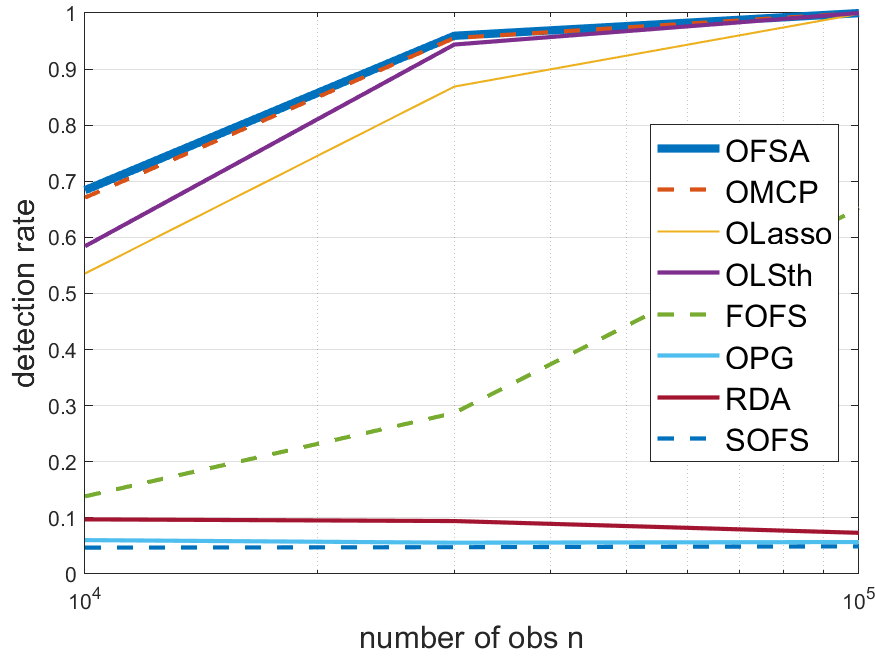}
\hspace{-2mm}&\includegraphics[width=0.3\linewidth]{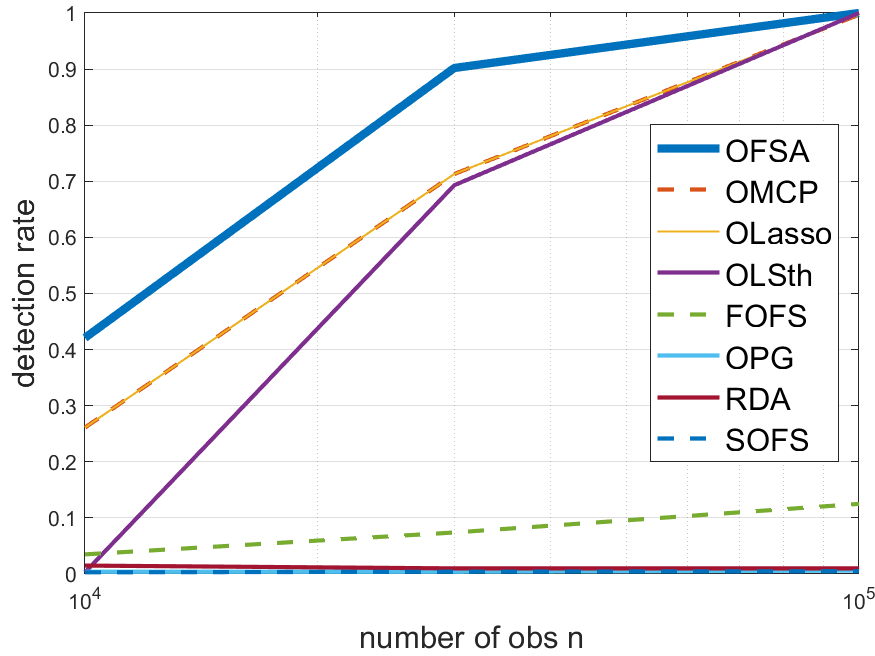}\\
$p=100, \beta=1$ &$p=1000, \beta=1$ &$p=10000, \beta=1$\\
\includegraphics[width=0.3\linewidth]{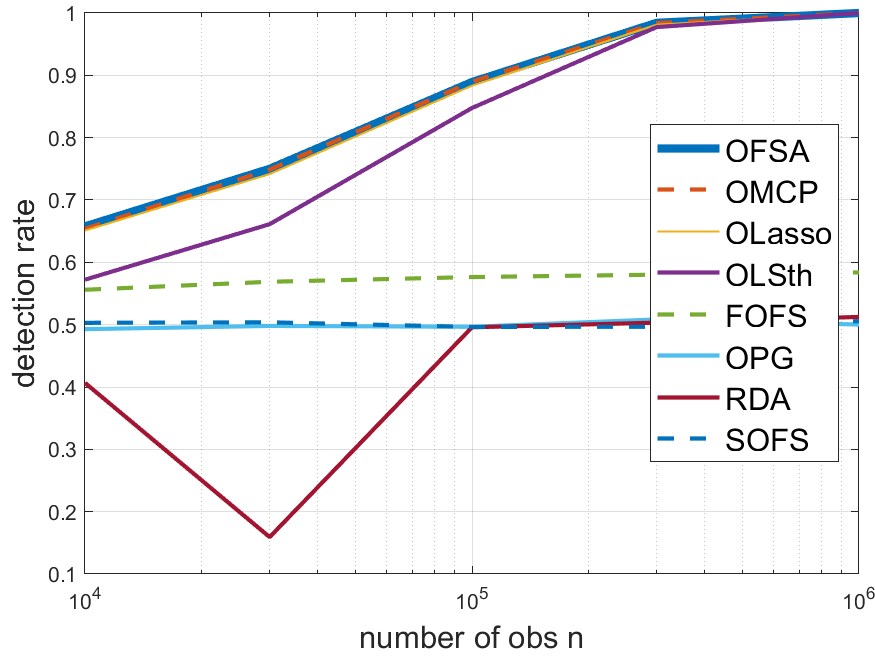}
\hspace{-2mm}&\includegraphics[width=0.3\linewidth]{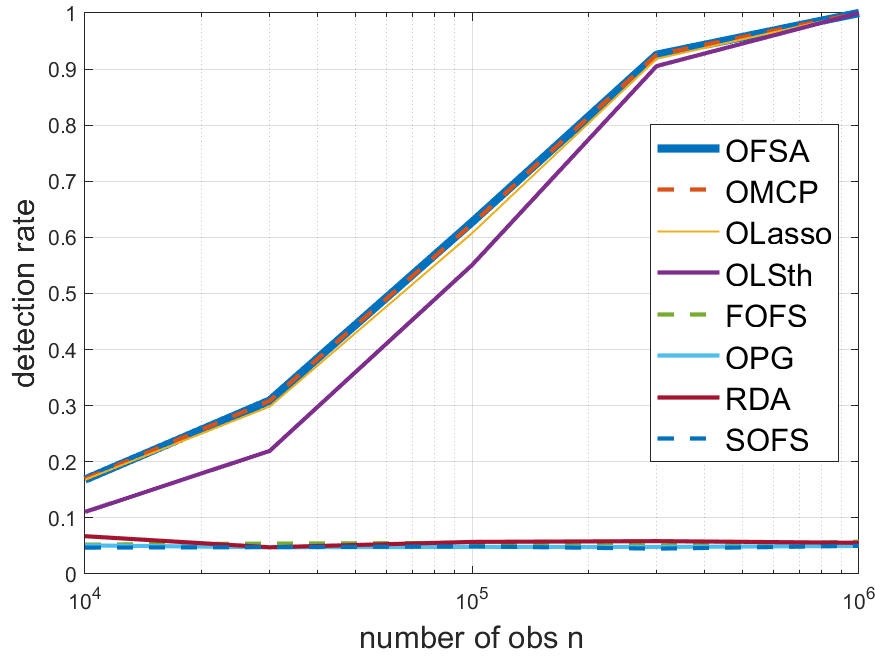}
\hspace{-2mm}&\includegraphics[width=0.3\linewidth]{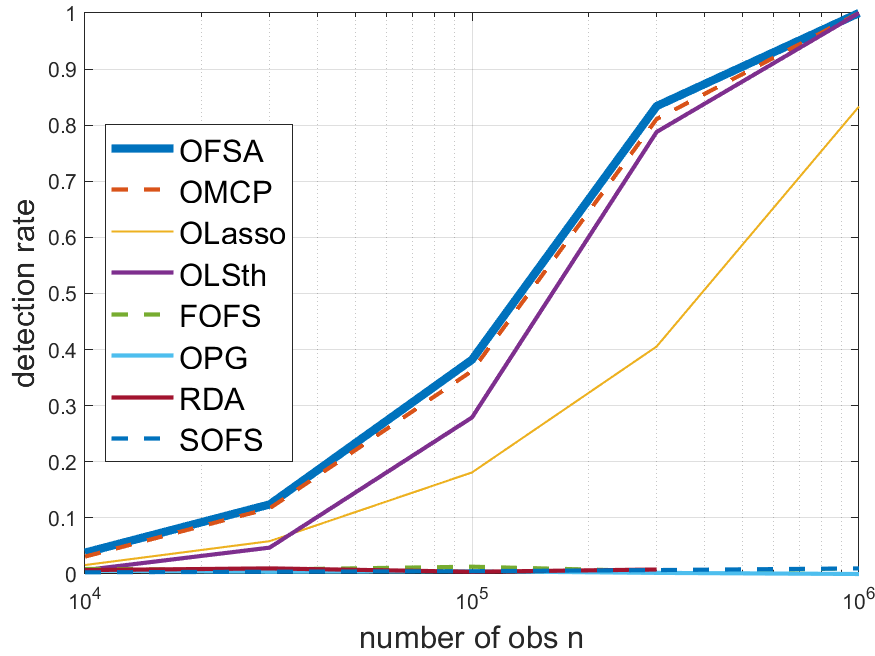}\\
$p=100, \beta=0.01$ &$p=1000, \beta=0.01$ &$p=10000, \beta=0.01$\\
\end{tabular}
\caption{Plot of the variable detection rate (DR) vs $n$ in classification, for $k=50$ true variables. Top: strong signal $\beta=1$, bottom: weak signal $\beta=0.01$.} \label{fig:plotclf_dr_n}
\end{figure}

The results, shown as the average of 100 independent runs, are presented in Figures \ref{fig:plot_dr_n} and \ref{fig:plotclf_dr_n} for $k=50$ and Tables \ref{tab:DR_RMSE_small} and \ref{table:DR&AUC} for $p=1000, k=100$.
The results from Figures \ref{fig:plot_dr_n} and \ref{fig:plotclf_dr_n} are also shown in detail in Tables S1 and S2 of the supplementary material.
Compared to the batch learning method Lasso, in regression, the running averages online methods enjoy low memory complexity. 
The larger datasets ($n=10^6$) cannot fit in memory and we cannot obtain the experimental results for Lasso in this case. 
For the proposed methods, the input is the running averages rather than the entire data matrix. 
The memory complexity for running averages is $\mathcal{O}(p^2)$, which is better than $\mathcal{O}(np)$ for batch learning in the setting of $n > p$. 

Table \ref{tab:minn_p} shows the minimum sample size $n$ observed to obtain a detection rate $DR>99\%$ for $k=50$ true variables. We see that all proposed running averages-based methods perform quite well, with the OFSA doing a little better for strong signals. 
The existing online methods we tested (RDA, OPG, SADMM, FOFS, and SOFS) never reached a $DR>99\%$ in our experiments.

From the numerical experiments, we can conclude that none of the existing online methods we tested (RDA, OPG, SADMM, FOFS, and SOFS) performs very well in true feature recovery. Only the offline Lasso and the proposed running averages-based methods can recover the true signal. When the signal is weak ($\beta=0.01$), although the running averages methods need a large sample size $n$ to recover the weak true signal, they outperform the batch learning methods and the other online methods in our experiments. 

\begin{table}[ht]
\centering
\caption{Minimum $n$ to obtain a detection rate $DR>99\%$ for $k=50$ true variables.}
\label{tab:minn_p}
\scalebox{0.8}{
\begin{tabular}{c|ccccccc}
\hline
 &\multicolumn{7}{c}{Regression} \\
$p$  &TSGD &SADMM &OLSth &OFSA &OMCP &OLasso/Elnet\\
\hline
\multicolumn{7}{l}{strong signal $\beta=1$}  \\
\hline
100 &- &- &$10^3$ &$10^3$ &$10^4$ &$10^4$\\
1000 &- &- &$3\cdot 10^3$ &$10^3$ &$10^4$ &$10^5$\\
10000 &- &- &$3\cdot 10^4$  &$3\cdot 10^3$ &$10^5$ &$3\cdot 10^4$\\
\hline
\multicolumn{7}{l}{weak signal $\beta=0.01$} \\
\hline
100  &- &- &$3\cdot 10^5$ &$3\cdot 10^5$ &$3\cdot 10^5$ &$3\cdot 10^5$\\
1000  &- &- &$10^6$ &$10^6$ &$10^6$ &$10^6$\\
10000  &- &- &$10^6$ &$10^6$ &$10^6$ &$10^6$\\
\hline
 &\multicolumn{7}{c}{Classification} \\
$p$  &FOFS &OPG&RDA&OLSth&OFSA&OMCP&OLasso\\
\hline
\multicolumn{7}{l}{strong signal $\beta=1$}  \\
\hline
100 &- &- &-  &$10^5$ &$3\cdot 10^4$ &$10^5$ &$10^5$\\
1000 &- &- &- &$10^5$ &$10^5$ &$10^5$&$10^5$\\
10000 &- &- &- &$10^5$ &$10^5$ &$10^5$&$10^5$\\
\hline
\multicolumn{7}{l}{weak signal $\beta=0.01$} \\
\hline
100  &- &- &- &$10^6$&$10^6$&$10^6$&$10^6$\\
1000  &- &- &-&$10^6$&$10^6$&$10^6$&$10^6$\\
10000 &- &- &-&$10^6$&$10^6$&$10^6$&-\\
\hline
\end{tabular}}
\end{table}

In prediction, most methods do well except in regression the existing methods (Lasso, TSGD, and SADMM) don't work well when the signal is strong. 
In contrast, the proposed running averages perform very well in prediction regardless of whether the signal is weak or strong, in both regression and classification.

Finally, we know that the computational complexity for obtaining the model from the running averages does not depend on the sample size $n$, but the time to update the running averages, shown as RAVE in Tables \ref{tab:DR_RMSE_small} and \ref{table:DR&AUC}, does increase linearly with $n$. 
Indeed, we observe in Tables \ref{tab:DR_RMSE_small} and \ref{table:DR&AUC} that the running time of OFSA and OLSth does not have significant changes. 
However, because of the need to tune the penalty parameters in OLasso, OElnet, and OMCP, it takes more time to run these algorithms. 
The computational complexity for traditional online algorithms will increase with sample size $n$. 
This is especially true for OPG, RDA, and SADMM, which take a large amount of time to tune the parameters to select $k$ features. 
When the sample size $n$ is very large, running these algorithms takes more than a few days to run 100 times.

\begin{figure}[ht]
\centering
\hspace{-2mm}
\includegraphics[width=0.33\linewidth]{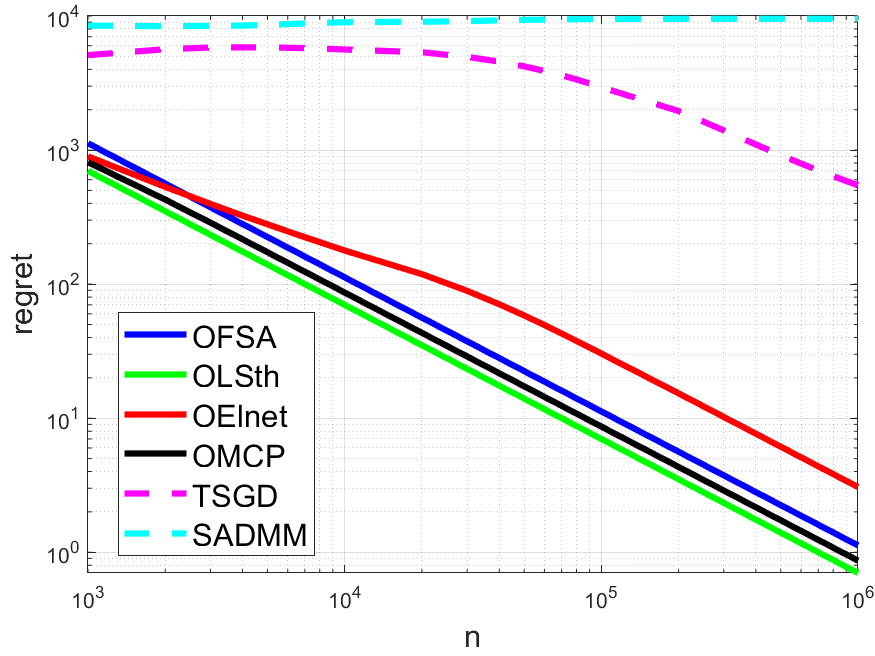}
\hspace{-2mm}
\includegraphics[width=0.33\linewidth]{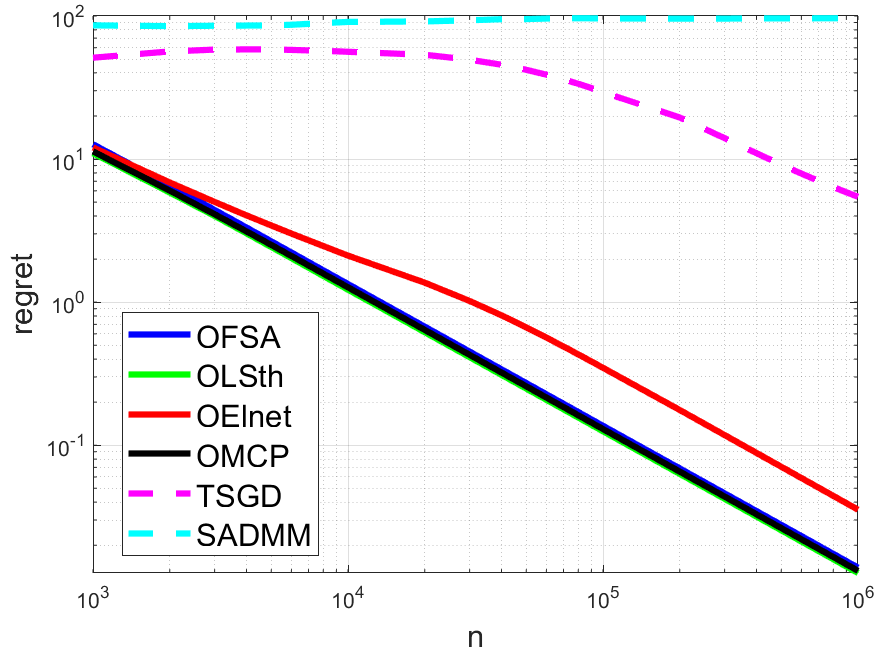}
\hspace{-2mm}
\includegraphics[width=0.33\linewidth]{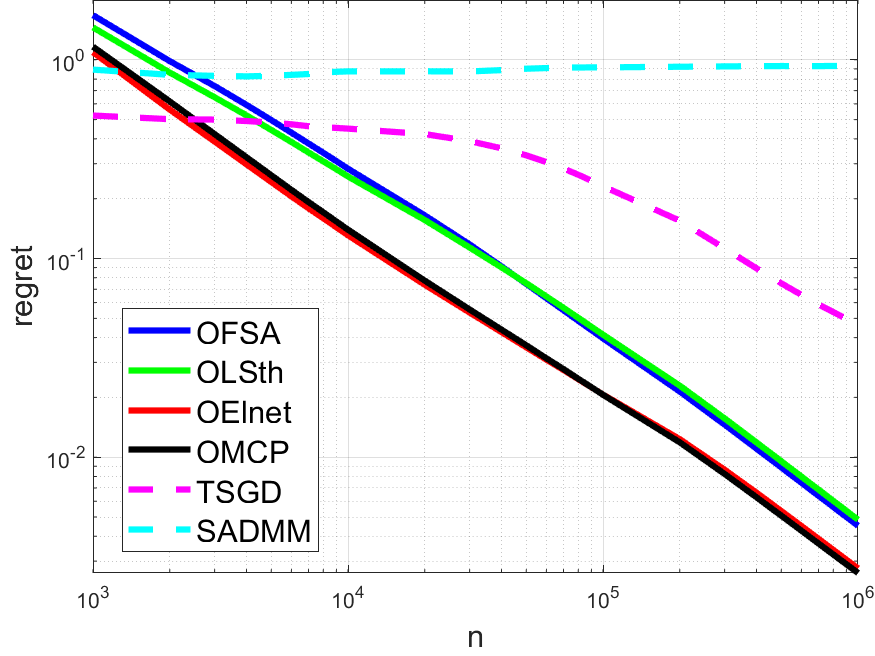}
\caption{Regret figures are presented for TSGD, SADMM, and running averages-based methods, averaged over 20 runs. Left: strong signal ($\beta = 1$), middle: medium signal ($\beta = 0.1$), right: weak signal ($\beta = 0.01$).}\label{fig:logreg}
\end{figure}

\subsubsection{Regret Analysis}
In this section, we present results about the regret of the different online methods in regression settings. 
In conventional online learning, theoretical analyses of the upper bound for regret were studied in \cite{hazan2007logarithmic} and \cite{zinkevich2003online}. 
Here, we focus on comparing the regret of the running averages-based online algorithms with the existing online methods. 

Figure \ref{fig:logreg} shows the curve of the regret for $\beta=1$(left), $\beta=0.1$(middle), $\beta = 0.01$(right). 
The sample size $n$ varies from 1000 to $10^6$. 
The slopes can be compared to see differences in convergence rates.
The regret of the SADMM method does not converge when the number of selected features is restricted to be at most $k$. 
The convergence rate for the running averages-based methods is close to $O(n^{-1})$. 
TSGD also seems to have the same convergence rate but starts with a plateau where the regret does not converge. 
\subsection{\bf Model Adaptation Experiments}
This section presents two simulations for linear regression models where the coefficients drift in time. 
In the first one, the data was generated as discussed in Section \ref{subsec:simdata} with the parameters $p = 100$ and $k = 10$, but here we assume that each nonzero $\beta_j$ varies as the observations are presented:
\begin{equation}\label{coef:vary}
\beta_{ij} = a \cos\{2\pi\frac{(i - 100j)}{T}\} + b, \text{ } j = \overline{1, k},  \text{ } i = \overline{1,T},
\end{equation}
in which $T$ is an unknown period. The values $a = 5, b = 5$, and $T = 1000$ were used in these simulations.
For each time, 1000 observations were generated. 
The running averages were updated with the model adaptation equation \eqref{adapt:eq}, where the adaptation rate was $\alpha_n = 0.01$.
\begin{figure}[ht]
\center
     \includegraphics[width=0.4\linewidth]{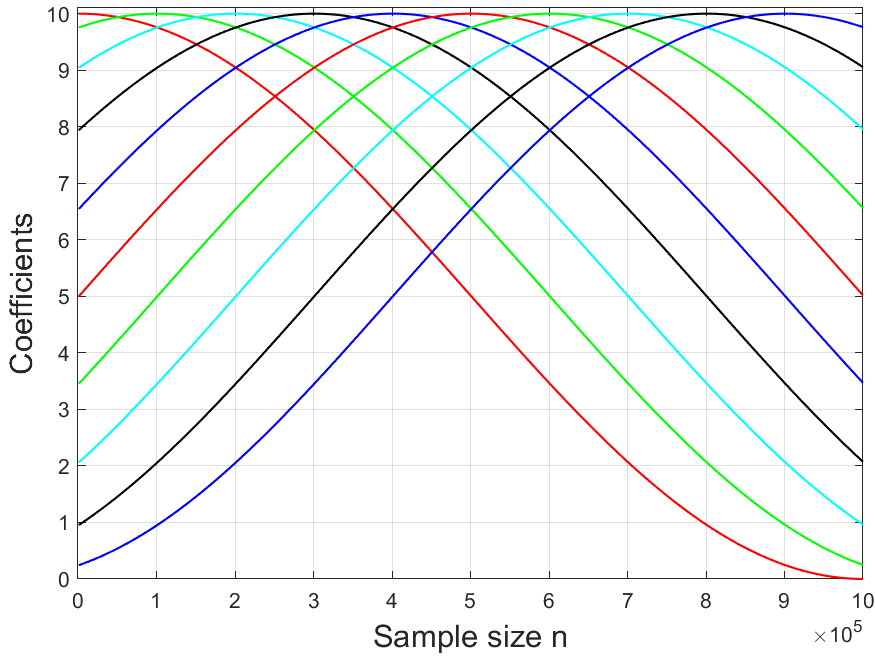}
     \includegraphics[width=0.4\linewidth]{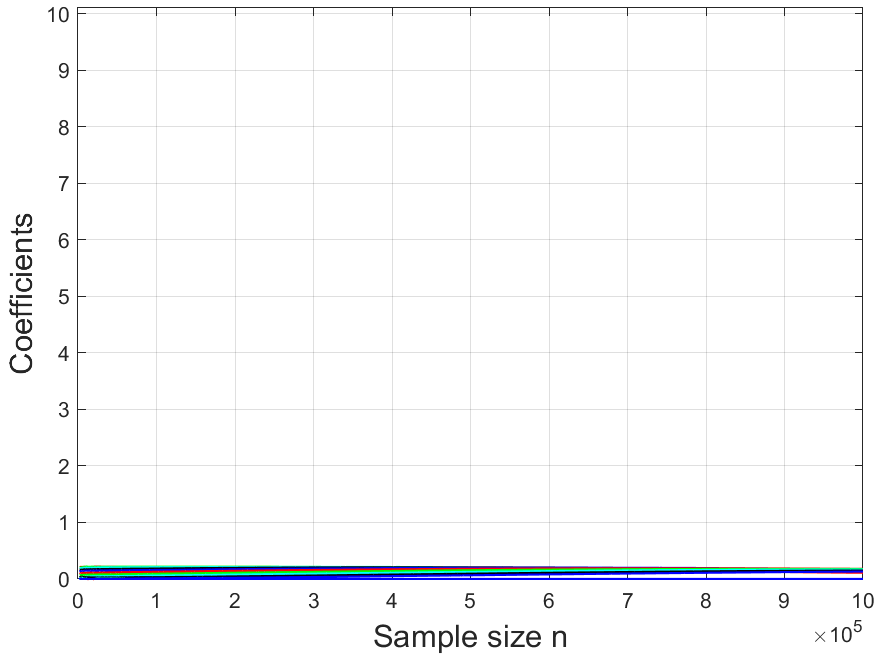}
     \includegraphics[width=0.4\linewidth]{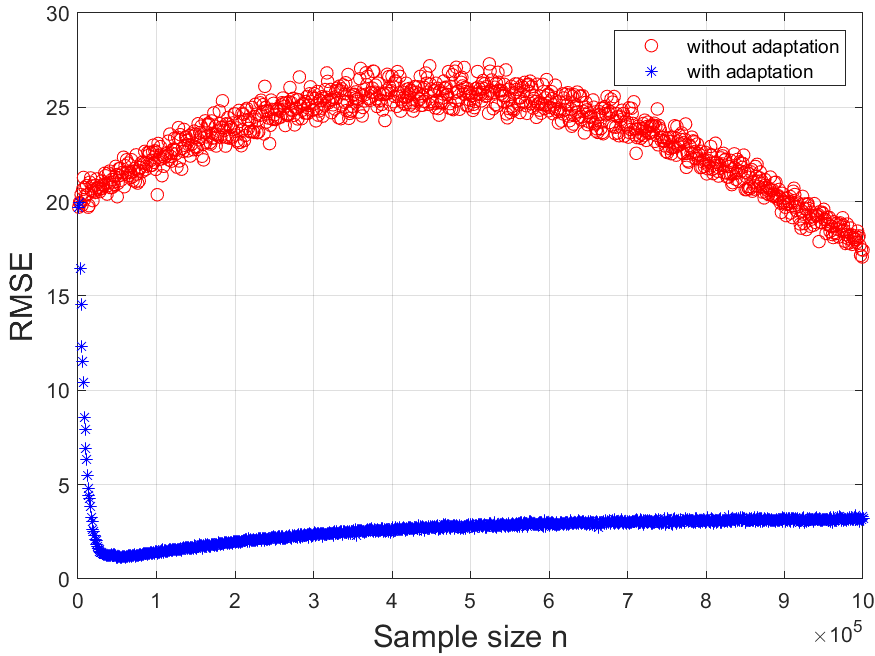}
     \includegraphics[width=0.4\linewidth]{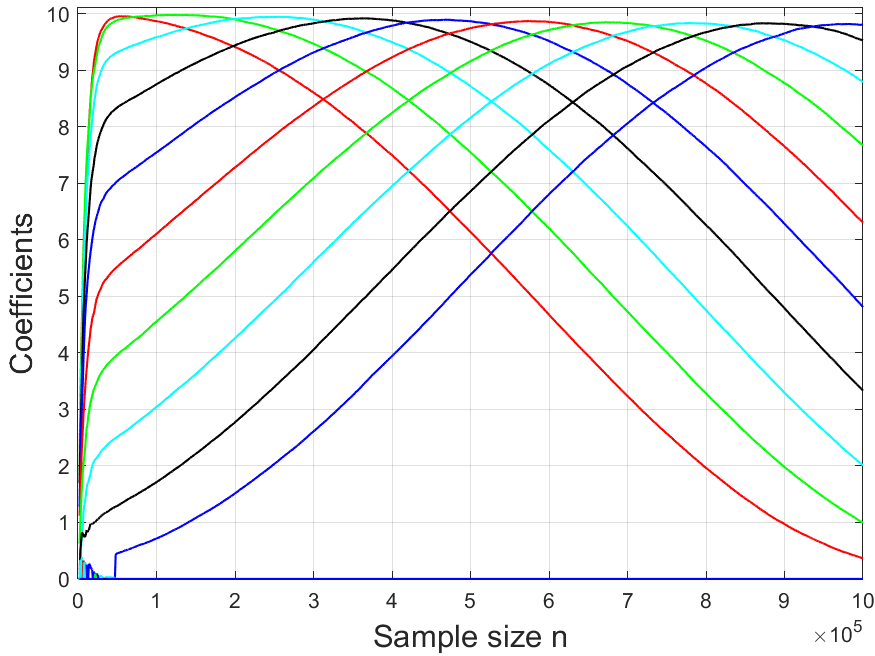}
\caption{Model adaptation experiment. Upper left: true signal. Upper right: estimated parameters without adaptation. Bottom left: RMSE for prediction. Bottom right: estimated parameters with adaptation.}\label{fig:modada}
\end{figure}
%%%%%%%%%%%%%%%%%%%%%%%%%%%%%%%%%%%%%%%%%%%%%%%%%%%%%%%%%%%%%%%

An example run is shown in Figure \ref{fig:modada}. 
One can see that our model adaptation method can follow the varying coefficients and perform better in prediction than without model adaptation. 
Table \ref{tab:modada} shows the RMSE for the last few hundred time steps, averaged over 20 independent runs. 
One can see that the RMSE with model adaptation is closer to the best RMSE possible $1.0$ and the prediction without model adaptation is quite poor.

\begin{table}[htb] 
\centering
\caption{\small RMSE for models with and without adaptation, averaged over 20 independent runs.}\label{tab:modada}
\scalebox{1.25}{
\begin{tabular}{ccc}
\hline
           &With adaptation &Without adaptation\\
\hline
Example 1& 3.163(0.005)&18.98(0.036)\\
Example 2& 3.021(0.008)&18.99(0.086)\\
\hline
\end{tabular}}
\end{table}

A second numerical experiment simulates a high dimensional dynamic pricing and demand problem \citep{qiang2016dynamic}. 
We assume that the demand $D_t$ follows a linear combination of price and the other covariates as
\vspace{-1mm}
\begin{equation*}
D_t = \beta_0 + \gamma p_t + \bx_t\bbeta_t + \epsilon_t, \text{ } t = 1, 2, \cdots,  
\end{equation*}  
in which $\gamma \in \mathbb{R}$ is the coefficient of the price $p_t$ at time $t$, and $\bbeta_t \in \mathbb{R}^{p-1}$ is the parameter vector for the other covariates, and $\gamma < 0$ in the model. 
The parameters $\gamma$, $\beta_0$, $\bbeta_t, t=1,2,...,$ are unknown to the seller and need to be estimated. 
Here we assume $\bbeta_t$ is sparse and varying with time. 
The above equation is commonly used in the economic community to model the relationship between the demand and the price. 
More details about the demand-price model can be found in \cite{qiang2016dynamic}.

The true price parameter was chosen as $\gamma = -0.5$, and $p_t \sim \mathcal{U}[10, 20]$. For the other covariates ($\beta_{tj}, j = 2, 3, \cdots, k$), we still used equation (\ref{coef:vary}), with $a = 5, b = 5, T = 2000$, and $\beta_{tj}=0$ for $j\in \{k+1,...,p\}$. For each time $t$, we generated 200 observations and used again the model adaptation rate $\alpha_n = 0.01$.

This simulation tries to find the relationship between demand and price in a varying marketplace. 
The marketplace is assumed to vary slowly over a very long time. 
This simulation setting is more complex than \cite{qiang2016dynamic} because we considered continuous varying coefficients as well as true feature recovery in our setting. 
However, we did not discuss the theoretical analysis here, which is left for a future study.  

%%%%%%%%%%%%%%%%%%%%%%%%%%%%%%%%%%%%%%%%%%%%%%%%%%%%%%%%%%%%%%%%%%%%%%%
\begin{figure}[ht]
\center
     \includegraphics[width=0.4\linewidth]{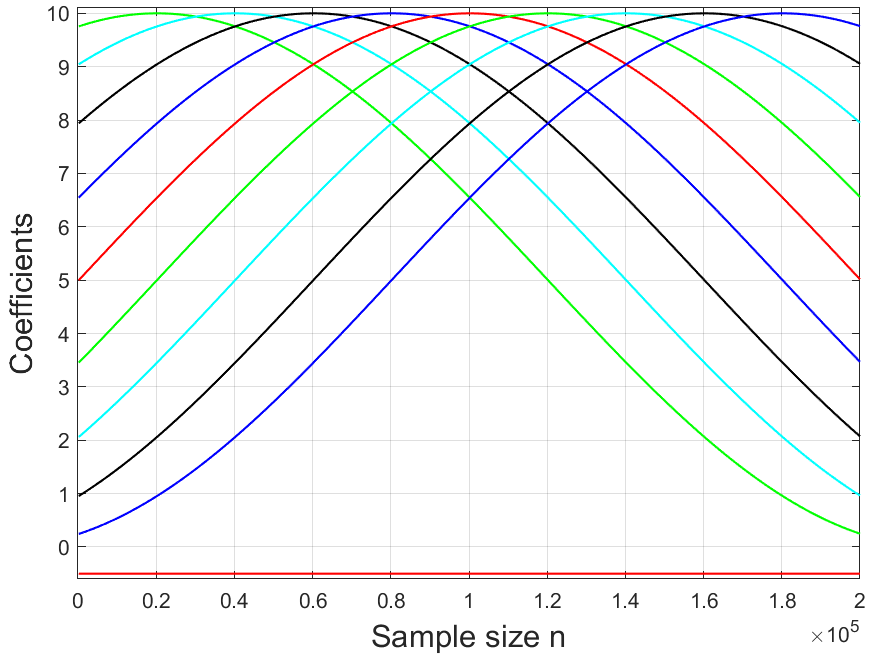}
     \includegraphics[width=0.4\linewidth]{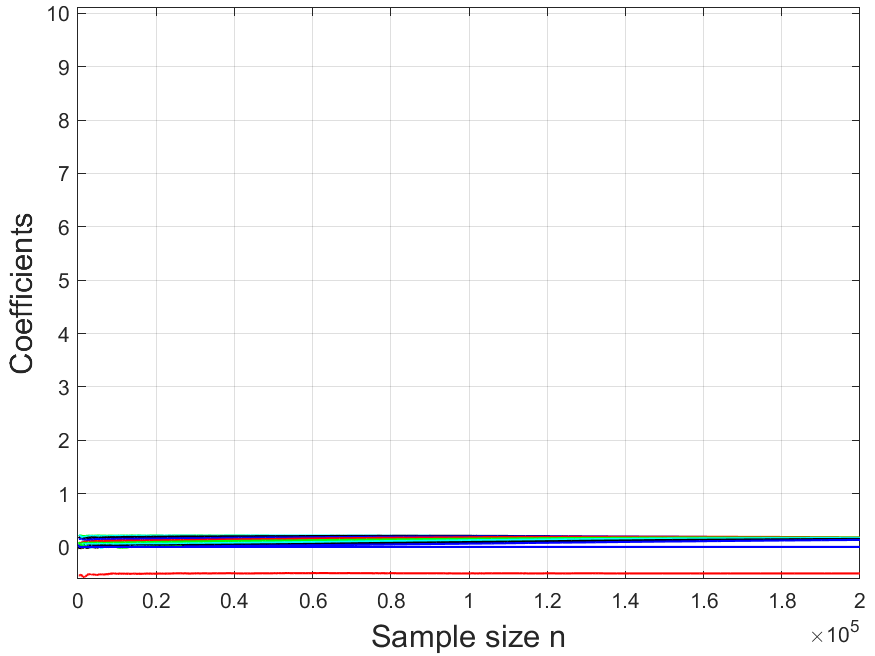}
     \includegraphics[width=0.4\linewidth]{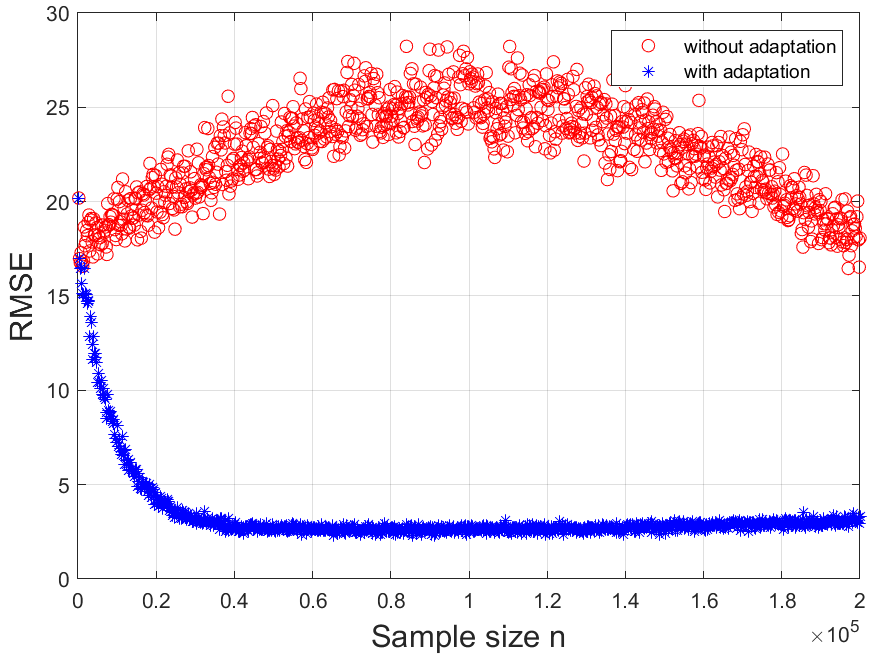}
     \includegraphics[width=0.4\linewidth]{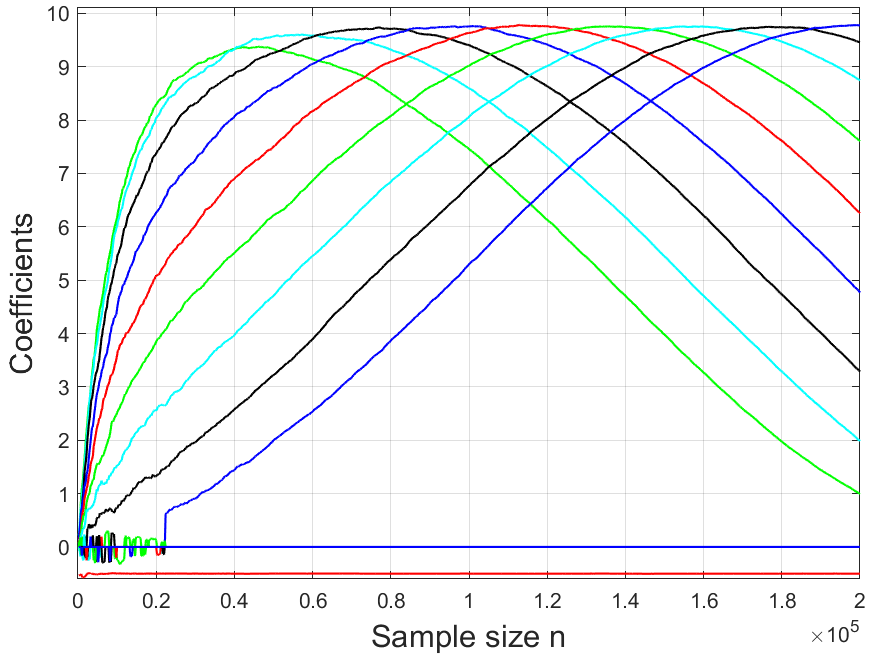}
   \caption{Model adaptation for dynamic pricing with feature selection. Upper left: true signal. Upper right: estimated parameters without adaptation. Bottom left: RMSE for prediction. Bottom right: estimated parameters with adaptation.}\label{fig:modadab}
\end{figure}
%%%%%%%%%%%%%%%%%%%%%%%%%%%%%%%%%%%%%%%%%%%%%%%%%%%%%%%%%%%%%%%%%%%%%%

The results of the dynamic pricing simulation are shown in Figure \ref{fig:modadab}. 
One can see again that the model adaptation works quite well in following the drifting coefficients, and the RMSE is much smaller than without adaptation. 
One can also see the power of the running averages for model selection in practice. 
In the plot of the estimated coefficients with adaptation, the smallest non-zero coefficient (blue line) oscillates between being in the model and being zero until sufficient data is available, then it is permanently added to the model.

\subsection{\bf Real Data Analysis}
\label{sec:realdata}

For real data analysis, the running averages-based methods were applied to some real-world datasets such as the Year Prediction MSD, Gisette, and Dexter datasets from the UCI Machine Learning Repository \citep{uci}, and two image datasets. 

The Year Prediction MSD, with 90 features and 463,715 observations, is about the prediction of the release year of a song from audio features. 
For this dataset, the linear model was extended to a quadratic model using running averages: new features were generated as products of all pairs of the 90 features, obtaining a 4185-dimensional feature vector. 

The Wikiface \citep{rothe2015dex, rothe2018deep} dataset is about age estimation from a single image. 
Age estimation is a regression problem, as age has a continuous range of values. 
The dataset contains 53,040 face images of actors from Wikipedia and their ages. 
The faces are cropped and resized to $224\times 224$ pixels. From each face image, a 4096-dimensional feature vector is extracted using the pre-trained VGG-16 \citep{simonyan2014very} convolutional neural network (CNN). 

\begin{table*}[tb]
\centering
\caption{\small{Real data results shown as average $R^2$ for regression and average AUC for classification. 
}}\label{table:realdata}
\scalebox{0.9}{
\begin{tabular}{lcccccccc}
\hline
Regression Data  &$n$  &$p$  &OLSth  &OFSA   &Lasso  &TSGD  &SADMM &MLP\\
\hline
WIKIFace                              &53k    &4096   &{\bf 0.547}  &0.545  &0.503  &0.304  &0.503 &0.388\\
Year Prediction MSD                   &464k   &90     &{\bf 0.237}  &{\bf 0.237}  &{\bf 0.237} &0.017  &0.183 &-3.62\\
Year Pred. MSD online &464k  &4185   &{\bf 0.303}  &0.298  &-         &-          &-    &-67.3\\
\hline
\hline
Classification Data            &$n$   &$p$    &OLSth  &OFSA   &Lasso  &FOFS  &SOFS &MLP\\
\hline
Gisette  &7k  &5000   &0.990    &{\bf 0.997}   &0.993   &0.566   &0.502 &0.995\\
Dexter   &600   &20000 &0.936    &{\bf 0.971}   &0.940   &0.499   &0.499 &0.891\\
ImageNet ObjDet &91m &2048 &{\bf 0.916} &{\bf 0.916} &- &0.901 &0.902 &0.910 \\
\hline                    
\end{tabular}}
\end{table*}
The results of the above datasets are shown as the average of 20 random splits of 87.5\% training and 12.5\% test data for the WikiFace dataset, and 80\% training and 20\% test data for the Year Prediction MSD, Gisette, and Dexter datasets. 

The ImageNet ObjDet dataset is about detecting objects in the images of the ImageNet dataset \citep{deng2009ImageNet}. 
The ImageNet dataset contains 1.3 million training images, 50,000 validation images, and 50,000 test images from 1000 object classes. 
About 50\% of the training images and all the validation images have annotations with the object locations as bounding boxes. 
The \verb|resnet50_swsl| model from \cite{yalniz2019billion} was used to extract features, where the final classification layer and the average pooling layer were removed. 
This model obtains from an image of size $h\times w$ a grid of $\sim h/32\times w/32$  feature vectors of size $2048$ each. 
Using this data we considered the binary classification problem of detecting the centers of the object bounding boxes, independent of the object's class.
The 2048 feature vectors at the locations corresponding to the centers of the object bounding boxes are the positive examples and all other feature vectors at a distance of at least 3 pixels from the positive locations are negatives.
All annotated images of the ImageNet training set were used for training, and the validation set for testing.
This way a training set with 615k positives and 90.3 million negatives was obtained, and a test set with 80k positives and 1.02 million negatives.
This dataset has 730GB, so it could not be stored in the computer memory and had to be regenerated every time it was used.

A multi-layer perceptron (MLP) with one hidden layer with 128 hidden nodes and no feature selection was also implemented, to evaluate what a non-linear model can accomplish when it is trained online (only one pass through the data).

For each method, multiple models were trained using various values of the tuning parameters and sparsity levels $k$.
Then the parameter combination with the largest average test $R^2$, average test AUC over 20 random splits or largest test AUC (for the ImageNet ObjDet data) is reported in Table \ref{table:realdata}.

From Table \ref{table:realdata} one can see that OLSth and OFSA perform best in both regression and classification and the other online feature selection methods perform quite poorly, especially on the Gisette and Dexter classification datasets where a small set of features have to be selected. On the ImageNet ObjDet dataset, all 2048 features are relevant, which is probably why the FOFS and SOFS obtained much better AUCs than on Gisette and Dexter. MLP did a good job on the classification datasets but lagged in regression.
Moreover, the offline Lasso, TSGD, and SADMM cannot handle large datasets such as the ObjDet data or the Year Prediction MSD data with pairwise interactions, and on the small size datasets, the offline Lasso, TSGD, and SADMM have a smaller $R^2$ than the methods in the proposed novel online framework. 

In contrast, our running averages-based methods can not only be used to build the non-linear model on Year Prediction MSD but also obtain a better $R^2$ than the linear model without using pairwise interactions.

%%%%%%%%%%%%%%%%%%%%%%%%%%%%%%%%%%%%%%%%%%%%%%%%%%%%%%%%%%%
\section{\bf Discussion}\label{sec6:discuss}
This paper introduced a complete framework for online learning based on running averages (RAVEs). 
The RAVEs are defined to replace the data to compute the gradient. The procedures to standardize the data in the running averages are introduced together with a series of feature selection algorithms based on them. Additionally, the running averages-based feature selection algorithms are introduced for binary classification with imbalanced data, which is very useful for computer vision and beyond.

In contrast to the standard online methods, the proposed framework can be used for model selection, in the sense that different models with different sparsity levels can be built at the same time, without seeing the data again. 
This is especially useful when more complex models are desired to be extracted from the data as the number of observations increases.

The running averages-based methods minimize the same loss function as their offline counterparts. However, they enjoy much lower computation complexity in the low-dimensional case. They enjoy good convergence rates and can provably recover the support of the true signal with high probability. Such theoretical support recovery guarantees have been proved for the offline penalized methods such as regression with the $\ell_1$, SCAD, MCP, or the Elastic Net, and these guarantees naturally extend to the corresponding methods in our framework. Moreover, this paper proves support recovery guarantees for OLSth and OFSA as well as upper regret bounds for OLS and OLSth, which are not presented in the existing literature. 

Numerical experiments have demonstrated that the running averages-based methods outperform conventional online learning algorithms and batch learning methods in prediction and feature selection. Moreover, the regret of the running average methods diminishes faster than the conventional online algorithms. 

The running averages-based methods could have a wide variety of applications, for example for detecting environmental changes and for recommendation systems. 
One of their main advantages is that they could detect and recover a very weak signal given sufficiently many observations.  

However, we also need to pay attention to the weaknesses of the running averages-based methods, as they have limitations on computational and space complexities when addressing high-dimensional datasets, i.e., the case of $p >> n$, or $p \to \infty$ with $n \to \infty$.
The memory complexity for the running averages methods is $\mathcal{O}(p^2)$ and the computational complexity is $O(np^2)$. 
A very large $p$ will be an issue since the running averages would not fit in the computer memory in this case. 

%\bmhead{Supplementary information}
\section*{Supplementary information}
The supplementary material contains the proofs of the theoretical guarantees from Section \ref{sec4:theory} and tables for some results from Section \ref{sec:realdata}.

%\bmhead{Acknowledgments}
\section*{Acknowledgments}
Barbu was partially supported by DARPA ARO grant W911NG-16-1-0579. 

\section*{Declarations}

\subsection*{Conflict of interest}
The authors declare that they have no conflict of
interest.

\subsection*{Funding}
This work was partially supported by DARPA ARO grant W911NG-16-1-0579. 

\subsection*{Data availability}
The real datasets used in this paper are publicly
available as specified in Section \ref{sec:realdata}.

\subsection*{Code availability}
The code will be made publicly available on \url{https://github.com/barbua/} upon paper acceptance.

\subsection*{Word Count}
The total word count is 8464.

\subsection*{Authors' contributions}
Sun and Barbu contributed to the study's conception and design. 
Sun proved the theoretical results.
Material preparation, data collection, and analysis were performed by Sun, Wang and Zhu. 
The first draft of the manuscript was written by Sun and was revised by Sun and Barbu. 
All authors have read and approved the final manuscript.

%% if your bibliography is in bibtex format, uncomment commands:
\bibliographystyle{asa.bst}
\bibliography{references}% Bibliography file (usually '*.bib')
\end{document}